\documentclass{article}

\usepackage{arxiv}

\usepackage{microtype}
\usepackage{graphicx}
\usepackage{subfigure}
\usepackage{booktabs}

\usepackage{amsmath}
\usepackage{amsfonts}
\usepackage{amsthm}
\usepackage{booktabs}
\usepackage{algorithm}
\usepackage{algorithmic}
\newtheorem{theorem}{Theorem}
\usepackage{multirow}

\usepackage{hyperref}
\usepackage{xr}

\newtheorem{manuallemmainner}{Lemma}
\newenvironment{manuallemma}[1]{%
  \renewcommand\themanuallemmainner{#1}%
  \manuallemmainner
}{\endmanuallemmainner}

\title{Predictive Coding Can Do Exact Backpropagation on Convolutional 
and Recurrent Neural Networks}

\author{
Tommaso Salvatori\thanks{Equal contributors} \\
Department of Computer Science\\
University of Oxford, UK\\
\texttt{tommaso.salvatori@cs.ox.ac.uk} \\
\And
Yuhang Song\footnotemark[1] \\
Department of Computer Science\\
University of Oxford, UK\\
\texttt{yuhang.song@some.ox.ac.uk} \\
\And
Thomas Lukasiewicz \\
Department of Computer Science\\
University of Oxford, UK\\
\texttt{thomas.lukasiewicz@cs.ox.ac.uk}\\
\And
Rafal Bogacz \\
MRC Brain Network Dynamics Unit \\
University of Oxford, UK \\
\texttt{rafal.bogacz@ndcn.ox.ac.uk}
\And
Zhenghua Xu \\
State Key Laboratory of Reliability and Intelligence of Electrical Equipment \\
Hebei University of Technology, Tianjin, China \\
\texttt{zhenghua.xu@hebut.edu.cn} \\
}

\begin{document}
\maketitle

\begin{abstract}
Predictive coding networks (PCNs) are an influential model for information processing in the brain. They have appealing theoretical interpretations and offer a single mechanism that accounts for diverse perceptual phenomena of the brain. On the other hand, \emph{backpropagation} (\emph{BP}) is commonly regarded to be the most successful learning method in  modern machine learning. Thus, it is exciting that recent work formulates \emph{inference learning} (\emph{IL}) that trains PCNs to approximate BP. However, there are several remaining critical issues: (i) IL is an approximation to BP with unrealistic/non-trivial requirements, (ii) IL approximates BP in single-step weight updates; whether it leads to the same point as BP after the weight updates are conducted for more steps is unknown, and (iii) IL is computationally significantly more costly than BP. To solve these issues, a variant of IL that  is strictly equivalent to BP in fully connected networks has been proposed.  In this work, we build on this result by showing that it also holds for more complex architectures, namely,  convolutional neural networks and (many-to-one) recurrent neural networks. To  our knowledge, we are the first to show that a biologically plausible algorithm is able to exactly replicate the accuracy of BP on such complex  architectures, bridging the existing gap between IL and BP,  and setting an unprecedented performance for PCNs, which can now be considered as efficient alternatives to~BP.
\end{abstract}


\section{Introduction}

\emph{Predictive coding networks} (\emph{PCNs}) are an influential model for describing information processing in the brain~\cite{rao1999predictive}.
They have appealing theoretical interpretations, such as via free-energy minimization \cite{bogacz2017tutorial,friston2003learning,friston2005theory,whittington2019theories} and probabilistic models \cite{whittington2017approximation},
and offer a single mechanism that accounts for diverse perceptual phenomena observed in the brain, such as repetition-suppression \cite{auksztulewicz2016repetition}, illusory motions \cite{lotter2016deep,watanabe2018illusory}, bistable perception \cite{hohwy2008predictive,weilnhammer2017predictive}, and attentional modulation of neural activity \cite{feldman2010attention,kanai2015cerebral}.
Though PCNs were originally proposed for unsupervised learning patterns of the brain 
 \cite{rao1999predictive}, they were later also found to be usable
for supervised learning: in a number of experiments, 
 multi-layer perceptrons with non-linear activations were shown to be able to approximate backpropagation (BP) \cite{werbos1974new,rumelhart1986learning} and do supervised learning at least as good as BP \cite{whittington2017approximation}.
Supervised PCN's learning works briefly as follows: in a PCN, (i) run an optimization process (called \emph{inference}) until convergence, so that prediction errors will be propagated across neurons, like the error term is backpropagated by gradient descent, and (ii) update weights to minimize the propagated prediction errors.
Being based on a special process of inference, such learning is also called  \emph{inference learning}~(\emph{IL}).
By recent work \cite{millidge2020predictive}, IL~can approximate BP also for more complex network structures, along arbitrary computation graphs. 
Seeing that IL (as a plausible model for  learning in the brain) can  approximate BP  (as the most successful learning method in modern machine learning) is of crucial importance for both the machine learning and the neuroscience~\mbox{community}.

However, despite such progress on the relationship between IL and BP, there are still some critical issues regarding complex architectures, as for IL's approximation to BP to hold, inference needs to have converged to an equilibrium \cite{millidge2020predictive}. This requirement is non-trivial: the convergence of inference (especially in practical non-linear models) is not easy empirically, nor is it guaranteed theoretically.
Moreover, even if the above two requirements are satisfied, this approximation only holds for a single step of weight update. Whether there exists an approximation of the error that converges to a bounded value after an infinite number of steps of weights updates is unknown.
Furthermore, inference in IL is also computationally substantially more expensive: IL's approximation to BP at a single step of a weights update requires inference to be conducted for sufficiently many  steps until convergence. For example, suppose that inference converges after $T\,{=}\,128$ steps (which is the common case for a model with two  layers), then IL is around $128$ times slower than BP, as one inference iteration is computationally about as costly as one iteration of BP \cite{millidge2020predictive}.

To solve these problems and bridge the gap between IL and BP, a variant of IL, called Z-IL, has been shown to perform the \emph{exact} same weights update of BP on fully connected networks \cite{song2020can}. 
This work, however, does not show whether this result carries over to complex architectures, such as convolutional and recurrent neural networks (CNNs and RNNs), which are the most used ones in practice.

In the present work, we positively answer this open question. We define the equations of IL for convolutional architectures and many-to-one recurrent networks, and mathematically prove that Z-IL is able to exactly replicate the dynamics of BP on these two models. In \cite{Lillicrap20}, the authors claim that no one in the machine-learning community has been able to train high-performing deep networks on difficult tasks such as classifying the objects in ImageNet photos using any algorithm other than BP. In this work, we thus close this gap. In fact, we show that Z-IL is, to our knowledge, the first biologically-inspired learning algorithm that is able to obtain an accuracy as good as the one obtained by BP on complex tasks such as large-scale image recognition on convolutional networks. 
We also analyze theoretically and experimentally 
 the running time of 
this learning algorithm on  CNNs and RNNs. Particularly, 
 we train multiple CNNs and RNNs on different tasks and datasets, and show that Z-IL can be up to orders of magnitude faster compared to IL, while only showing a little overhead compared to BP. Hence, the conclusion that Z-IL is significantly cheaper than~IL holds regardless of the complexity of the model.

Our work will be of significant importance for the machine learning community, as it shows that exact BP can indeed be implemented in complex architectures present in the brain. On the other hand, our work
will also be  highly benefitial for the neuroscience community, as activity patterns observed in the cortex may now be studied using BP, which has been criticized previously due to its biological implausibility.

The main contributions of this paper are briefly as follows. 
\begin{itemize}
    \item We define convolutional and recurrent versions of the IL and Z-IL algorithms, thus encoding these two architectures in the framework of predictive coding.
    \item We mathematically prove that the update rules of BP and Z-IL are equivalent in convolutional and recurrent networks, generalizing the existing result for fully connected networks \cite{song2020can}. Therefore, PCNs match the performance of BP on complex models.
    \item We theoretically and experimentally analyze the running time of Z-IL, IL and BP on CNNs and RNNs. We show that Z-IL is not only equivalent to BP in terms of performance, but also comparable in terms of efficiency. Furthermore, it is several orders of magnitude faster than~IL.
\end{itemize}

\section{Preliminaries}

We now highlight the similarities and differences between artificial neural networks (ANNs) trained with BP, and PCNs  trained with IL.
Both models have a \emph{prediction} and a \emph{learning} stage. The former has no supervision signal, and the goal is to use the current parameters to make a prediction. During learning, the  supervision signal is present, and the goal is to update the current parameters.
To make the dimension of variables explicit, we denote vectors  with a bar (e.g.,  $\overline{x}\,{=}\,(x_{1},x_{2},\ldots, x_n)$ and $\overline{0}\,{=}\,(0,0,\ldots,0)$).


\begin{table*}[t]
  \caption{Notation for ANNs and PCNs.}
  \label{tb:notations}
  \vspace*{0.5ex}\centering
  \resizebox{1.0\textwidth}{!}{
    \begin{tabular}{ccccccccccccccc}
    \toprule
    \cmidrule(r){1-13}
    \rotatebox[origin=c]{0}{} & \rotatebox[origin=c]{90}{\shortstack{\hspace*{0ex}Value-node\hspace*{-1ex}\\ \hspace*{0ex}activity\hspace*{-1ex}}} & \rotatebox[origin=c]{90}{\shortstack{\hspace*{-2ex}Predicted\hspace*{-1ex}\\\hspace*{-2ex}value-node\hspace*{-1ex}\\\hspace*{-1ex}activity\hspace*{-1ex}}} & \rotatebox[origin=c]{90}{\shortstack{Error term\hspace*{-1ex}\\ or node\hspace*{-1ex}}} & \rotatebox[origin=c]{90}{\shortstack{\hspace*{3ex}Weight}} & \rotatebox[origin=c]{90}{\shortstack{Objective\hspace*{-1ex}\\function\hspace*{-1ex}}} & 
    \rotatebox[origin=c]{90}{\shortstack{Activation\hspace*{-1ex}\\function\hspace*{-1ex}}} & \rotatebox[origin=c]{90}{\shortstack{\hspace*{2ex}Layer size\hspace*{-1ex}}} & \rotatebox[origin=c]{90}{\shortstack{Number of\hspace*{-1ex}\\layers\hspace*{-1ex}}} & \rotatebox[origin=c]{90}{\shortstack{Input\hspace*{-1ex}\\signal\hspace*{-1ex}}} & \rotatebox[origin=c]{90}{\shortstack{Supervision\hspace*{-1ex}\\signal\hspace*{-1ex}}} & \rotatebox[origin=c]{90}{\shortstack{Learning rate\ \hspace*{-1ex}\\for weights\hspace*{-1ex}}} &
    \rotatebox[origin=c]{90}{\shortstack{Convolution-\hspace*{-1ex}\\ al kernel\hspace*{-1ex}}} &
    \rotatebox[origin=c]{90}{\shortstack{\hspace*{-2ex}Integration\hspace*{0ex}\\ \hspace*{-2ex}step for\hspace*{0ex}\\ \hspace*{-2ex}inference\hspace*{0ex}}}  \vspace*{-0.75ex}\\
    \midrule
    {ANNs} & $y_{i}^{\scriptscriptstyle {l}}$ & -- & $\delta_{i}^{\scriptscriptstyle {l}}$ & $w_{i,j}^{\scriptscriptstyle {l}}$ & $E$ & 
    \multirow{2}{*}{$f$} & \multirow{2}{*}{$n^{\scriptscriptstyle {l}}$} & \multirow{2}{*}{$l_{\text{max}}$} & \multirow{2}{*}{$s^{\text{in}}_{i}$} & \multirow{2}{*}{$s^{\text{out}}_{i}$} & \multirow{2}{*}{$\alpha$} & $\rho$ & -- \\
    {PCNs} & $x^{\scriptscriptstyle {l}}_{i,t}$ & $\mu^{\scriptscriptstyle {l}}_{i,t}$ & $\varepsilon^{\scriptscriptstyle {l}}_{i,t}$ & $\theta_{i,j}^{\scriptscriptstyle {l}}$ & $F_{t}$ 
    &  &  &  &  &  & & $\lambda$ & $\gamma$ \\
    \bottomrule
  \end{tabular}\vspace*{-1ex}
  }
\end{table*}


\subsection{ANNs trained with BP}

ANNs consist of a sequence of layers \cite{rumelhart1986learning}.
Following~\cite{whittington2017approximation}, we invert the usual notation used to number the layers. Particularly, we index the output layer by $0$ and the input layer by $l_{\text{max}}$.
Furthermore, we call~$y_{i}^{\scriptscriptstyle {l}}$ the input to the $i$-th node in the $l$-th layer. Hence, we have
\begin{equation}
y^{\scriptscriptstyle {l}}_{i} = {\textstyle\sum}_{j=1}^{n^{\scriptscriptstyle {l+1}}} w^{\scriptscriptstyle {l+1}}_{i,j} f ( y^{\scriptscriptstyle {l+1}}_{j} ), 
\label{eq:bp-forward}
\end{equation}
where $f$ is the activation function, $w^{\scriptscriptstyle {l+1}}_{i,j}$ is the weight that connects the $j$-{th} node in the $(l\,{+}\,1)$-{th} layer to the $i$-{th} node in the $l$-{th} layer, and $n^{\scriptscriptstyle {l}}$ is the number of neurons in layer~$l$.
To make the notation lighter, we avoid considering bias values.  
However, in the supplementary material, we show that every result of this work extends to the more general case of networks with bias values.

The prediction phase works as follows: given an input vector
$\overline{s}^{\text{in}}= (s^{\text{in}}_{1},\ldots,s^{\text{in}}_{n^{l_{\text{max}}}})$, every node of the input layer is set to the corresponding node of $\overline{s}^{\text{in}}$. Every $ y^{\scriptscriptstyle {l}}_{i}$ is then computed iteratively according to Eq.~\eqref{eq:bp-forward}.

\smallskip 
\noindent\textbf{Learning: }%
Given 
a pair $(\overline{s}^{\text{in}},\overline{s}^{\text{out}})$ from the training set,  
$\overline{y}^{\scriptscriptstyle 0}= (y^{\scriptscriptstyle 0}_{1},\ldots,y^{\scriptscriptstyle 0}_{n^{\scriptscriptstyle 0}})$ is computed via Eq.~\eqref{eq:bp-forward} from $\overline{s}^{\text{in}}$
as input and compared with  $\overline{s}^{\text{out}}$ via 
the following objective function $E$:
\begin{equation}
E = \mbox{$\frac{1}{2}$} {\textstyle\sum}_{i=1}^{n^{\scriptscriptstyle {0}}} ( s^{\text{out}}_{i} - y^{\scriptscriptstyle {0}}_{i} ) ^{2}.
\label{eq:bp-error}
\end{equation}
Hence, the weight update performed by BP is
\begin{equation}
\Delta w^{\scriptscriptstyle {l+1}}_{i,j} 
= -\alpha\cdot {\partial E}/{\partial w^{\scriptscriptstyle {l+1}}_{i,j}}
=\alpha\cdot \delta^{\scriptscriptstyle {l}}_{i}
f ( y^{\scriptscriptstyle {l+1}}_{j} ), 
\label{eq:bp-update-param}
\end{equation}
where $\alpha$ is the learning rate, and $\delta^{\scriptscriptstyle {l}}_{i} = {\partial E}/{\partial {y}^{\scriptscriptstyle {l}}_{i}}$ is the \textit{error term}, 
given as follows: 
\begin{equation}
\delta^{\scriptscriptstyle {l}}_{i} = \begin{cases} 
s^{\text{out}}_{i} - y^{\scriptscriptstyle {0}}_{i} & \mbox{\!\!if } l = 0\,;\\
f' ( y^{\scriptscriptstyle {l}}_{i} ) {\textstyle\sum}_{k=1}^{n^{\scriptscriptstyle {l-1}}} \delta^{\scriptscriptstyle {l-1}}_{k} w^{\scriptscriptstyle {l}}_{k,i} & \mbox{\!\!if } l \in \lbrace 1,\ldots, l_{\text{max}}{-}1 \rbrace\,.\!\!\!
\end{cases}
\label{eq:delta-recursive}
\end{equation}

\subsection{PCNs trained with IL}

\begin{algorithm}[t]
    \caption{Learning one training pair $(\overline{s}^{\text{in}},\overline{s}^{\text{out}})$ with IL}\label{algo:IL}
    \begin{algorithmic}[1]
    \REQUIRE $\overline{x}^{\scriptscriptstyle {l_{\text{max}}}}_{0}$ is fixed to $\overline{s}^{\text{in}}$; $\overline{x}^{\scriptscriptstyle {0}}_{0}$ is fixed to $\overline{s}^{\text{out}}$.
    \FOR{$t=0$ to $T$ (included)}
        \FOR{each neuron $i$ in each level $l$}
            \STATE Update $x^{\scriptscriptstyle {l}}_{i,t}$ to minimize $F_{t}$ via Eq.~\eqref{eq:pcn-dotx-il}\
            \IF{$t= T$}
                \STATE Update each  $\theta^{\scriptscriptstyle {l+1}}_{i,j}$ to minimize $F_{t}$ 
                via Eq.~\eqref{eq:pcn-update-param} \ \  
                {
                }
            \ENDIF
        \ENDFOR
    \ENDFOR
    \end{algorithmic}
\end{algorithm}

\begin{figure}
\centering
	\includegraphics[width=0.8\textwidth]{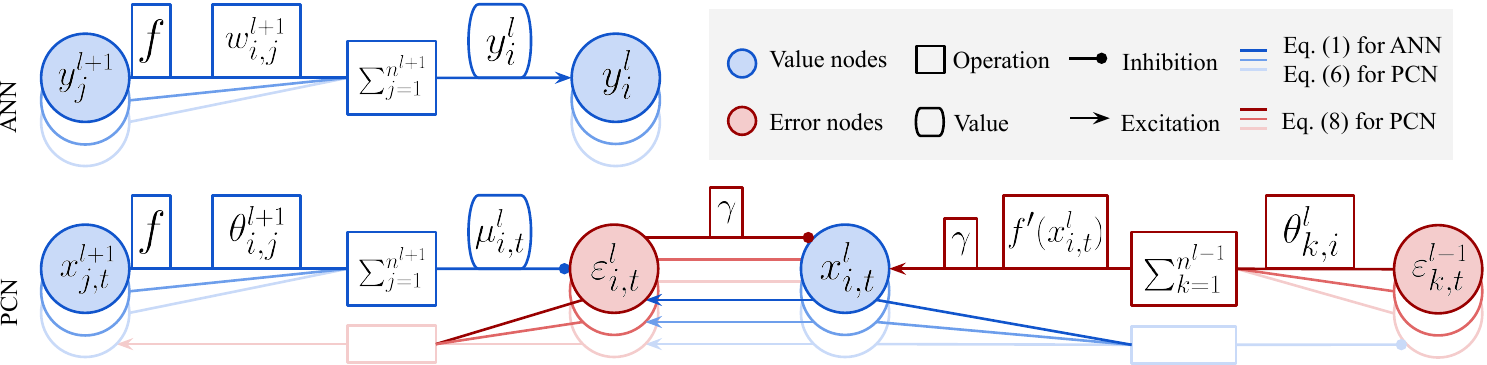}
 	\vspace*{-2ex}
	\caption{ANNs and PCNs trained with BP and IL, respectively.
    }\vspace*{-2ex}
	\label{fig:pcn-ann}
\end{figure}

PCNs  are a widely used model of information processing in the brain, originally developed for unsupervised learning.
It has recently been shown that, when a PCN is used for supervised learning, it closely approximates BP \cite{whittington2017approximation}.
It involves inferring the values of hidden nodes for a time $T$ before doing the weight update, hence the name IL. We call $t$ the time axis during inference.
A PCN contains value nodes, which are each associated with a corresponding prediction-error node.  
Differently from ANNs, which propagate the activity between value nodes directly, PCNs propagate the activity between value nodes $x^{\scriptscriptstyle {l}}_{i,t}$ via the error nodes $\varepsilon^{{\scriptscriptstyle l}}_{i,t}$:
\begin{equation}
\mu^{\scriptscriptstyle {l}}_{i,t} = {\textstyle\sum}_{j=1}^{n^{\scriptscriptstyle {l+1}}} \theta^{\scriptscriptstyle {l+1}}_{i,j} f ( x^{\scriptscriptstyle {l+1}}_{j,t} )
\mbox{ \ \ and \ \ }
\varepsilon^{\scriptscriptstyle {l}}_{i,t} = x^{\scriptscriptstyle {l}}_{i,t} - \mu^{\scriptscriptstyle {l}}_{i,t},
\label{eq:pcn-forward-varepsilon}
\end{equation}
where the $\theta^{\scriptscriptstyle {l+1}}_{i,j}$'s are the connection weights, 
paralleling $w^{\scriptscriptstyle {l+1}}_{i,j}$ in the described ANN, and 
$\mu^{\scriptscriptstyle {l}}_{i,t}$ is the prediction of $x^{\scriptscriptstyle {l}}_{i,t}$ based on 
the value nodes  
in a higher layer $x^{\scriptscriptstyle {l+1}}_{j,t}$. 
Thus, 
$\varepsilon^{\scriptscriptstyle {l}}_{i,t}$ computes the difference between the actual and the predicted~$x^{\scriptscriptstyle {l}}_{i,t}$.

The value node $x^{\scriptscriptstyle {l}}_{i,t}$ is modified so that the overall energy~$F_{t}$ in  $\varepsilon^{\scriptscriptstyle {l}}_{i,t}$ is minimized all the time:{
\begin{align}
F_{t} = 
{\textstyle\sum}_{l=0}^{l_{\text{max}}-1} {\textstyle\sum}_{i=1}^{n^{\scriptscriptstyle {l}}}
{ \mbox{$\frac{1}{2}$} ( \varepsilon^{\scriptscriptstyle {l}}_{i,t} ) ^2}\,. 
\label{eq:pcn-f}
\end{align}}%
\indent This way, $x^{\scriptscriptstyle {l}}_{i,t}$ tends to move close to $\mu^{\scriptscriptstyle {l}}_{i,t}$.
Such a process of minimizing $F_{t}$ by modifying all $x^{\scriptscriptstyle {l}}_{i,t}$ is called \textit{inference}, and it is running during both prediction and learning.
It minimizes $F_{t}$ by modifying $x^{\scriptscriptstyle {l}}_{i,t}$, following a unified rule for both stages: {
\begin{equation}
\!\Delta{x}^{\scriptscriptstyle {l}}_{i,t} 
= \begin{cases}
0  &  \!\!\mbox{if } l=l_{\text{max}} \\
\gamma\cdot ( -\varepsilon^{\scriptscriptstyle {l}}_{i,t} + f' ( x^{\scriptscriptstyle {l}}_{i,t} ) {\textstyle\sum}_{k=1}^{n^{\scriptscriptstyle {l-1}}} \varepsilon^{\scriptscriptstyle {l-1}}_{k,t} \theta^{\scriptscriptstyle {l}}_{k,i} ) & \!\!\mbox{if } 0 \,{<}\ l \,{<}\ l_{\text{max}}\\
\gamma \cdot ( -\varepsilon^{\scriptscriptstyle {l}}_{i,t} ) & \!\!\mbox{if } l=0\text{,}
\end{cases}
\label{eq:pcn-dotx-il}
\end{equation}}%
\noindent where $x^{\scriptscriptstyle {l}}_{i,t+1}\, {=}\,  x^{\scriptscriptstyle {l}}_{i,t} \,{+}\,\Delta{x}^{\scriptscriptstyle {l}}_{i,t}$, and $\gamma$ is the integration step for~$x^{\scriptscriptstyle {l}}_{i,t}$. 
Here, $\Delta{x}^{\scriptscriptstyle {l}}_{i,t}$ is different between prediction and learning  only for $l=0$, as the output value nodes $x^{\scriptscriptstyle {0}}_{i,t}$ are left unconstrained during prediction {(i.e., optimized by Eq.~\eqref{eq:pcn-dotx-il})} and are fixed to $s^{\text{out}}_{i}$ {(i.e., $\Delta{x}^{\scriptscriptstyle {0}}_{i,t}=0$)} during learning.

\smallskip 
\noindent\textbf{Prediction: }%
Given an input
$\overline{s}^{\text{in}}$, 
the value nodes $x^{\scriptscriptstyle {l_{\text{max}}}}_{i,t}$ in the input layer are set to 
$s^{\text{in}}_{i}$.
Then, all the error nodes $\varepsilon^{\scriptscriptstyle {l}}_{i,t}$ are optimized by the inference process and decay to zero as $t\,{\rightarrow}\,\infty$.
Thus, the value nodes $x^{\scriptscriptstyle {l}}_{i,t}$ converge to $\mu^{\scriptscriptstyle {l}}_{i,t}$,  the same values as  ${y}^{\scriptscriptstyle {l}}_{i}$ of the corresponding ANN with the same weights.

\smallskip 
\noindent\textbf{Learning: }%
Given 
a pair $(\overline{s}^{\text{in}},\overline{s}^{\text{out}})$ from the training set,  
the value nodes of both the input and the output layers are set to the training pair (i.e., $x^{\scriptscriptstyle {l_{\text{max}}}}_{i,t} = s^{\text{in}}_{i}$ and $x^{\scriptscriptstyle {0}}_{i,t} = s^{\text{out}}_{i}$); thus, 
\begin{equation}
\varepsilon^{\scriptscriptstyle {0}}_{i,t} = x^{\scriptscriptstyle {0}}_{i,t} - \mu^{\scriptscriptstyle {0}}_{i,t} = s^{\text{out}}_{i} - \mu^{\scriptscriptstyle {0}}_{i,t}.
\label{eq:pcn-forward-varepsilon-l0}
\end{equation}
Optimized by the inference process, the error nodes $\varepsilon^{\scriptscriptstyle {l}}_{i,t}$ can no longer decay to zero;
instead, they converge to values as if the errors had been backpropagated.
Once the inference converges to an equilibrium ($t \,{=}\,T$), where~$T$ is a fixed large number, a \textit{weight update} is performed.

The weights~$\theta^{\scriptscriptstyle {l+1}}_{i,j}$ are updated to minimize the same objective function~$F_{t}$;   thus,

\begin{equation}
\Delta \theta^{\scriptscriptstyle {l+1}}_{i,j} = -\alpha\cdot {\partial F_{t}}/{\partial \theta^{\scriptscriptstyle {l+1}}_{i,j}} 
=  \alpha\cdot \varepsilon^{\scriptscriptstyle {l}}_{i,t} f ( x^{\scriptscriptstyle {l+1}}_{j,t} ),
\label{eq:pcn-update-param}
\end{equation}
where $\alpha$ is the learning rate.
By Eqs.~\eqref{eq:pcn-dotx-il} and \eqref{eq:pcn-update-param}, all computations are local (local plasticity) 
in IL,
and the model can autonomously switch between prediction and learning via running inference, as the two phases aim to minimize the same energy function. This makes IL more biologically plausible than BP.
The learning of IL is summarized in Algorithm~\ref{algo:IL}.

\section{Convolutional Neural Networks (CNNs)}



\emph{Convolutional neural networks} (\emph{CNNs}) are biologically inspired networks with a connectivity pattern (given by a set of \emph{kernels}) that resembles the structure of animals' visual cortex. Networks with this particular architecture are widely used in image recognition tasks. A CNN is formed by a sequence of convolutional layers, followed by a sequence of fully connected ones. For simplicity of notation, in the body of this work, we only consider convolutional layers with one kernel. In the supplementary material, we show how to extend our results to the general case. We now recall the structure of CNNs and compare it against convolutional~PCNs.

\subsection{CNNs Trained with BP}

The learnable parameters of a convolutional layer are contained in different kernels. Each kernel $\bar \rho^{\scriptscriptstyle {l}}$ can be seen as a vector of dimension $m$,
which acts on the input vector $ f(\bar y^{\scriptscriptstyle {l}})$ using an operation ``$*$'', called \emph{convolution}, which is equivalent to a linear transformation with a sparse matrix $w^{\scriptscriptstyle {l}}$, whose non-zero entries equal to the entries of the kernel $\bar \rho^{\scriptscriptstyle {l}}$. This particular matrix is called doubly-block circulant matrix \cite{Sedghi20}. For every entry $\rho^{\scriptscriptstyle {l}}_a$ of a kernel, we denote by $\mathcal C^{\scriptscriptstyle {l}}_a$ the set of indices $(i,j)$ such that $w^{\scriptscriptstyle {l}}_{i,j} = \rho^{\scriptscriptstyle {l}}_a$.

Let $f(\bar y^{\scriptscriptstyle {l+1}})$ be the input of a convolutional layer with kernel $\bar{\rho}^{\scriptscriptstyle {l}}$. The output $\bar y^{\scriptscriptstyle {l}}$ can then be computed as in the fully connected case: it suffices to use Eq.~\eqref{eq:bp-forward}, where~$w^{\scriptscriptstyle {l}}$ is the doubly-block circulant matrix with parameters in $\overline \rho^{\scriptscriptstyle {l}}$. During the learning phase, BP updates the parameters of $\bar \rho^{\scriptscriptstyle {l+1}}$ according to the following equation:
\begin{equation}
\Delta \rho^{\scriptscriptstyle {l+1}}_{a} 
= -\alpha\cdot {\partial E}/{\partial \rho^{\scriptscriptstyle {l+1}}_{a}}
= {\textstyle\sum}_{(i,j) \in \mathcal C^{\scriptscriptstyle {l+1}}_a} \Delta w^{\scriptscriptstyle {l+1}}_{i,j}.
\label{eq:cnn_weight}
\end{equation}
The value $\Delta w^{l}_{i,j}$ can be computed using Eq.~\eqref{eq:bp-update-param}.

\subsection{Predictive Coding CNNs Trained with IL}

\begin{figure}[t]
    \centering
	\includegraphics[width=0.4\textwidth]{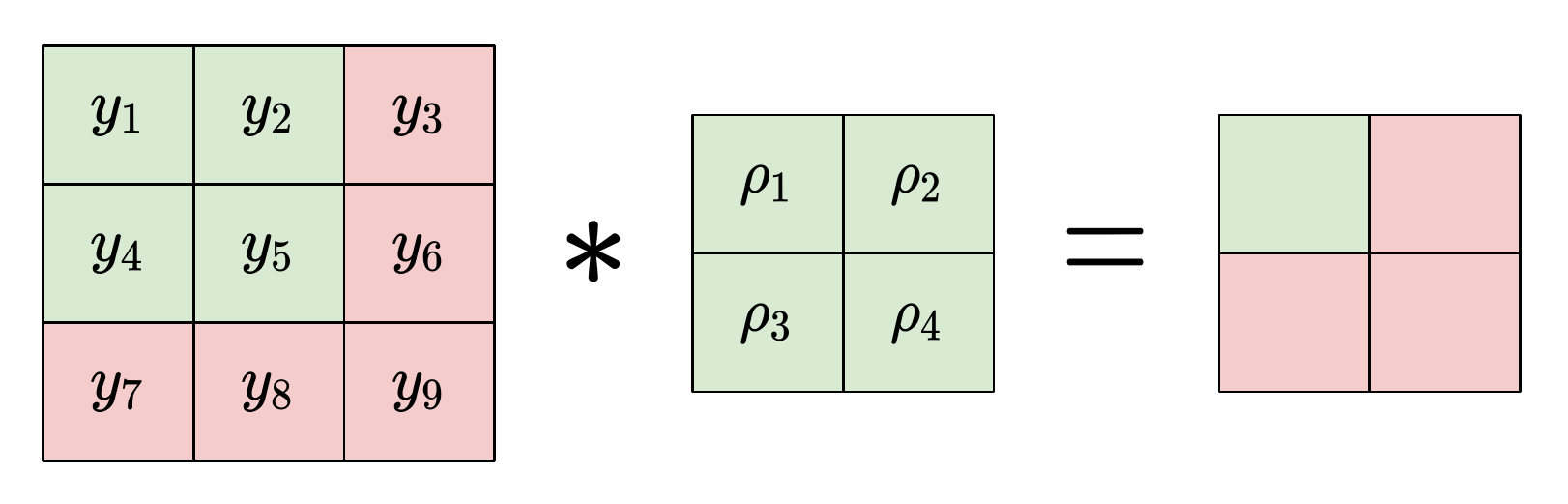}\vspace*{-1.5ex}
	\caption{An example of a \emph{convolution}.}\vspace*{-0.5ex}
	\label{fig:cnn}
\end{figure}

Given a convolutional network, we call  $\overline{\lambda}^{\scriptscriptstyle {l}}$ the kernels of dimension $k$, and $\theta^{\scriptscriptstyle {l}}$ the related double-block circular matrix, which describes the convolution operation. Note that $\overline{\lambda}^{\scriptscriptstyle {l}}$ mirrors $\overline \rho^{\scriptscriptstyle {l}}$ defined in CNNs, while $\theta^l$ mirrors $w^l$. The dynamics of the feedforward pass is the same as the one described in the fully connected case. Hence,  the quantities $\varepsilon^{\scriptscriptstyle {l}}_{i,t}$ and $ \Delta x_{i,t}^{\scriptscriptstyle {l}}$ are computed as in Eqs.~\eqref{eq:pcn-forward-varepsilon} and \eqref{eq:pcn-dotx-il}. The update of the entries of the kernels, on the other hand, is the following:
\begin{equation}
    \Delta {\lambda}_a^{\scriptscriptstyle {l+1}} = -\alpha \cdot \frac{\partial F_t}{\partial {\lambda}_a^{\scriptscriptstyle {l+1}}} =  {\textstyle\sum}_{(i,j) \in \mathcal C^{\scriptscriptstyle {l+1}}_a} \Delta \theta_{i,j}^{\scriptscriptstyle {l+1}}, 
    \label{eq:cnn_il}
\end{equation}
where $\Delta \theta_{i,j}^{\scriptscriptstyle {l}}$ is computed according to Eq.~\eqref{eq:pcn-update-param}.

Detailed derivations of  Eqs.~\eqref{eq:pcn-dotx-il}, \eqref{eq:pcn-update-param}, and \eqref{eq:cnn_il} are given in the supplementary material.

\begin{algorithm}[t]
    \caption{Learning one training pair $(\overline{s}^{\text{in}},\overline{s}^{\text{out}})$ with Z-IL}\label{algo:Z-IL}
    \begin{algorithmic}[1]
    \REQUIRE $\overline{x}^{\scriptscriptstyle {l_{\text{max}}}}_{0}$ is fixed to $\overline{s}^{\text{in}}$; $\overline{x}^{\scriptscriptstyle {0}}_{0}$ is fixed to $\overline{s}^{\text{out}}$.
    \REQUIRE $x^{\scriptscriptstyle {l}}_{i,0}=\mu^{\scriptscriptstyle {l}}_{i,0}$ for $l \in \lbrace 1, \ldots , l_{\text{max}}-1 \rbrace$, and $\gamma=1$.
    \FOR{$t=0$ to $l_{\text{max}}-1$ (included)}
        \FOR{each neuron $i$ in each level $l$}
            \STATE Update $x^{\scriptscriptstyle {l}}_{i,t}$ to minimize $F_{t}$ via Eq.~\eqref{eq:pcn-dotx-il}\
            \IF{$t= l$}
                \STATE Update each  $\theta^{\scriptscriptstyle {l+1}}_{i,j}$ to minimize $F_{t}$ 
                via Eq.~\eqref{eq:pcn-update-param} \ \  
                {
                }
            \ENDIF
        \ENDFOR
    \ENDFOR
    \end{algorithmic}
\end{algorithm}

\subsection{Predictive Coding CNNs Trained with Z-IL}

Above, we have described the training and prediction phases on a single point $s = (\bar s^{\text{in}},\bar s^{\text{out}})$ under different architectures. The training phase of IL on a single point runs for $T$ iterations, during which the inference of Eq.~\eqref{eq:pcn-dotx-il} is conducted, and $T$ is a hyperparameter that is usually set to different sufficiently large values  to get inference converged \cite{whittington2017approximation,millidge2020predictive}. So, the inference phase starts at $t\,{=}\,0$ and ends at $t\,{=}\,T$, which is also when the 
network 
parameters are updated  via 
Eqs.~\eqref{eq:pcn-update-param} and~\eqref{eq:cnn_il}. 

To show that IL is able to do exact BP on both the fully connected and convolutional layers of a CNN, we add constraints on the weights update of IL. These constraints produce a IL-based learning algorithm, called \emph{zero-diver\-gen\-ce}~IL~(\emph{Z-IL}) \cite{song2020can}, which  satisfies the local plasticity condition present in the brain and can autonomously switch between prediction and learning via running inference. We now show that it also produces \emph{exactly} the same weights update as BP on complex architectures such as CNNs.

\smallskip 
\noindent\textbf{Z-IL:} Let $M$ be a PCN model with $l_{\text{max}}$ layers. The inference phase runs for $T = l_{\text{max}}$ iterations. Instead of updating all the weights simultaneously at $t=T$, the parameters of every layer $\theta^{l+1}$ are updated at $t=l$. Hence, the prediction phase of Z-IL is equivalent to the one of IL, while the learning phase updates the parameters according to the  equation
\begin{equation}
\!\Delta{\theta}^{\scriptscriptstyle {l+1,t}}_{i,j} 
= \begin{cases}
0  &  \!\!\mbox{if } t \neq l \\
\alpha\cdot \varepsilon^{\scriptscriptstyle {l}}_{i}
f ( x^{\scriptscriptstyle {l+1}}_{j,t} ) &  \!\!\mbox{if } t = l.
\end{cases}
\label{eq:pcn-dotx}
\end{equation}%
We now show that, under a specific choice of hyperparameters, Z-IL is equivalent to BP on CNNs. Particularly, we add the following two conditions: $\varepsilon^l_{i,0} \,{=}\, 0$ for $l\,{>}\,0$, and $\gamma  \,{=}\, 1$. 

 The first condition can be obtained by setting $\bar x^l_{0} \,{=}\, \bar \mu^l_{0}$ for every $l\,{>}\,0$ at the start of inference. Considering the vector $\bar \mu^l_0$ is computed from $\bar x^{l+1}_0$ via  Eq.~\eqref{eq:pcn-forward-varepsilon}, this allows IL to start from a prediction stage  equivalent to the one of~BP. The second condition guarantees the propagation of the error during the inference phase to match the one of BP. Without it, 
 Z-IL would be equivalent to a variation of BP, where the weight updates of single layers would have different learning rates. 

The following theorem shows that Z-IL on convolutional PCNs is equivalent to BP on classical CNNs.

\begin{theorem}\label{thm:1}
    \label{ther:final-equal} 
    Let $M$ be a convolutional PCN trained with Z-IL with $\gamma \,{=}\, 1$ and $\varepsilon^l_{i,0} \,{=}\, 0$ for $l\,{>}\,0$, and let $M'$ be its corresponding CNN, initialized as $M$ and trained with BP.
   Then, given the same datapoint $s$ to both networks, we have 
   \begin{equation}
    \Delta \theta^{\scriptscriptstyle {l+1}}_{i,j} = \Delta w^{\scriptscriptstyle {l+1}}_{i,j}  \ \ \ \mbox{and} \ \ \ \Delta {\lambda}^{l+1}_i = \ \Delta \rho^{l+1}_i,
    \end{equation}
    for every $i,j,l \geq 0$.
\end{theorem}

\begin{proof}[Proof (Sketch)]
The proof is divided into three different steps. Each step proves a specific claim:

\smallskip 
\noindent\textbf{Step 1: } At $t=l$, we have $\varepsilon^l_{i,l} = \delta^l_i$ for every $l$.

\noindent This part of the proof is done by induction on the number of layers. We begin by noting that $w^l_{i,j} = \theta^l_{i,j}$, because of the same initialization of the network. This, together with the condition $\varepsilon^l_{i,0} = 0$ for $l>0$, shows the equivalence of the initial states.  The condition  $\gamma = 1$  guarantees that the error~$\varepsilon^l_{i,t}$ spreads as BP among different layers. 

\smallskip 
\noindent\textbf{Step 2: } We have  $\Delta \theta^{\scriptscriptstyle {l+1}}_{i,j} = \Delta w^{\scriptscriptstyle {l+1}}_{i,j}$ for every $i,j,l \geq 0$.

\noindent The result proven in the first step allows to substitute $\varepsilon^l_{i,l} = \delta^l_i$ in the following equations:
\begin{equation}
    \begin{split}
    & \Delta w^{\scriptscriptstyle {l+1}}_{i,j} = \alpha \cdot \delta^{l}_i f(y^{l+1}_{j}) \\
    & \Delta \theta^{\scriptscriptstyle {l+1}}_{i,j} = \alpha \cdot \varepsilon^{l}_i f(x^{l+1}_{j,l}).
    \end{split}
    \label{eq:weight_proof}
\end{equation}
The missing piece $f(y^{l+1}_{j}) \,{=}\, f(x^{l+1}_{j,l})$ is proved in the supplementary material. Note that the proof of this claim is sufficient to show the equivalence between BP and PCN in fully connected networks.  The following last step proves this equivalence also in convolutional layers.

\smallskip 
\noindent\textbf{Step 3: } We have  $\Delta \lambda^{\scriptscriptstyle {l+1}}_{a} = \Delta \rho^{\scriptscriptstyle {l+1}}_{a}$ for every $a,l \geq 0$.

\noindent The laws that regulate the updates of the kernels are shown in Eqs.~\eqref{eq:cnn_weight} and~\eqref{eq:cnn_il}. In PCNs trained with Z-IL, these two equations are equal by the result shown in Step 2. 
\end{proof}

We have shown the main claims of this work about CNNs: (i) BP can be made biologically plausible when trained on CNNs, (ii) Z-IL is a learning algorithm that allows PCNs to perfectly replicate the dynamics of BP on complex models, and (iii)~Z-IL removes the unrealistic/non-trivial requirements of~IL. Note that the new conditions in Z-IL are a simple setup of some hyperparameters or an adjustment of the algorithm, while the requirements of IL (that the model can make a sufficiently perfect prediction, and inference can suffi\-cient\-ly converge) are significantly more difficult to control.

\section{Recurrent Neural Networks (RNNs)}

While CNNs achieve impressive results in computer vision tasks, their performance drops when handling data with sequential structure, such as natural language sentences. An example is sentiment analysis: given a sentence $S^{\text{in}}$ with words $(\overline s^{\text{in}}_1, \dots, \overline s^{\text{in}}_N)$, predict whether this sentence is positive, negative, or neutral. To perform classification and regression tasks on this kind of data, the last decades have seen the raise of recurrent neural networks (RNNs). Networks that deal with a sequential input and a non-sequential output are called \emph{many-to-one} RNNs. An example of such an architecture is shown in Fig.~\ref{fig:rnn}.
In this section, we show that the proposed Z-IL, along with our conclusions, can be extended to RNNs as well.
We first recall RNNs trained with BP, and then show how to define a recurrent PCN trained with IL. We conclude  by showing that the proposed Z-IL can also be carried over and scaled to RNNs, and that our equivalence conclusions still hold.

\subsection{RNNs Trained with BP}

An RNN for classification and regression tasks has three different weight matrices $w^x, w^h$, and $ w^y$, $N$ hidden layers of dimension $n$, and an output layer {of dimension} $n^{\text{out}}$. When it does not lead to confusion, we will alternate the notation between $k=\text{out}$ and $k=N+1$. This guarantees a lighter notation in the formulas. A sequential input $S^{\text{in}} = \{ \overline s^{\text{in}}_{1}, \dots , \bar s^{\text{in}}_{N} \}$ is a sequence of $N$ vectors of dimension $n^{\text{in}}$. The first hidden layer is computed using the first vector of the sequential input, while  the output layer is computed by multiplying the last hidden layer by the matrix $w^y$, i.e.,  $\bar y^{\text{out}} = w^y \cdot f(\bar y^N) $. The structure of the RNN with the used notation is summarized in Fig.~\ref{fig:rnn}. By assuming $\overline y^0 = \overline 0$, the local computations of the network can be written~as follows:
\begin{equation}
    \begin{split}
    & y^k_i = {\textstyle\sum}_{j=1}^{n} w^h_{i,j} f ( y^{\scriptscriptstyle {k-1}}_{j} ) + {\textstyle\sum}_{j=1}^{n^{\text{in}}} w^x_{i,j} s^{\text{in}}_{k,j}, \\
    & y^{\text{out}}_i = {\textstyle\sum}_{j=1}^{n} w^y_{i,j} f ( y^{\scriptscriptstyle N}_{j} ).
    \end{split}
    \label{eq:rnn-forward}
\end{equation}

\smallskip 
\noindent\textbf{Prediction: }%
Given a sequential value
$S^{\text{in}}$ as input, every $y^{\scriptscriptstyle {k}}_{i}$ in the RNN is computed via Eq.~\eqref{eq:rnn-forward}.

\smallskip 
\noindent\textbf{Learning: }%
Given a sequential value
$S^{\text{in}}$ as input, the output $\bar y^{\text{out}}$ is then compared with the label $\overline s^{\text{out}}$ using the objective function $E$ described by Eq.~\eqref{eq:bp-error}.
We now show how BP updates the weights of the three weight matrices. Note that $w^y$ is a fully connected layer that connects the last hidden layer to the output layer. We have already computed this specific weight update in Eq.~\eqref{eq:bp-update-param}:
\begin{equation}
    \Delta w^y_{i,j} = \alpha \cdot \delta^{\text{out}}_i f(y^N_j) \mbox{ \ \ with \ \ } \delta^{\text{out}}_i =  s^{\text{out}}_{i} - y^{\text{out}}_{i}.
\end{equation}
The gradients of $E$ relative to the single entries of 
$w^x$ and $w^y$ are the sum of the gradients at each recurrent layer $k$. Thus,
\begin{equation}
    \begin{split}
        & \Delta w^x_{i,j} = \alpha \cdot {\textstyle\sum}_{k=1}^N \delta^{k}_i s^{\text{in}}_{k,j} \\
        & \Delta w^h_{i,j} = \alpha \cdot {\textstyle\sum}_{k=1}^N \delta^{k}_i f(y^{k-1}_j).
    \end{split}
    \label{eq:weight_rnn_bp}
\end{equation}
The error term $\delta^{\scriptscriptstyle {k}}_{i} \,{ =}\, {\partial E}/{\partial {y}^{\scriptscriptstyle {k}}_{i}}$ is defined as in~Eq.~\eqref{eq:delta-recursive}:

\begin{equation}
    \delta^{\scriptscriptstyle {k}}_{i} = f' ( y^{\scriptscriptstyle {k}}_{i} ) {\textstyle\sum}_{j=1}^{n} \delta^{\scriptscriptstyle {k+1}}_{j} w^{\scriptscriptstyle {h}}_{j,i}.
    \label{eq:error-rnn}
\end{equation}
\subsection{Predictive Coding RNNs Trained with IL}

\begin{figure}
    \centering
	\includegraphics[width=0.3\textwidth]{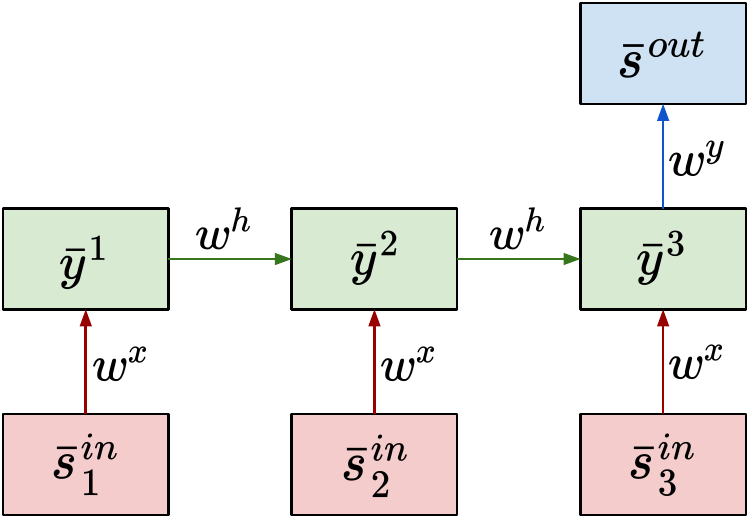}\vspace*{-1ex}
	\caption{An example of a \emph{many-to-one} RNN.}\vspace*{-1ex}
	\label{fig:rnn}
\end{figure}

We show how to define a recurrent PCN trained with IL.
Recurrent PCNs have the same layer structure as the network introduced in the previous section. Hence, by assuming $\overline x^0 = \overline 0$, the forward pass is given by as follows: 
\begin{equation}
    \begin{split}
    & \mu^k_{i,t} = {\textstyle\sum}_{j=1}^{n} \theta^h_{i,j} f ( x^{\scriptscriptstyle {k-1}}_{j,t} ) + {\textstyle\sum}_{j=1}^{n} \theta^x_{i,j} s^{\text{in}}_{k,j}, \\
    & \mu^{\text{out}}_{i,t} = {\textstyle\sum}_{j=1}^{n} \theta^y_{i,j} f ( x^{\scriptscriptstyle N}_{j,t} ).
    \end{split}
    \label{eq:rnn-forward-il}
\end{equation}
Here, $\theta^x$, $\theta^h$, and $\theta^y$ are the weight matrices paralleling $w^x$, $w^h$, and $w^y$, respectively. The $\mu^{\scriptscriptstyle {k}}_i$ and $x^{\scriptscriptstyle {k}}_i$ are defined as in the preliminaries. Again, error nodes computes the error between them $\varepsilon^{\scriptscriptstyle {k}}_{i,t} = x^{\scriptscriptstyle {k}}_{i,t} - \mu^{\scriptscriptstyle {k}}_{i,t}$. During the inference phase, the value nodes $x^{\scriptscriptstyle {k}}_{i,t}$ are updated to minimize the energy function in Eq.~\eqref{eq:pcn-f}. During the learning phase, this update is done via:
\begin{equation}
\!\Delta{x}^{\scriptscriptstyle {k}}_{i,t} = \begin{cases}
\gamma\cdot ( -\varepsilon^{\scriptscriptstyle {k}}_{i,t} + f' ( x^{\scriptscriptstyle {k}}_{i,t} ) {\textstyle\sum}_{j=1}^{n} \varepsilon^{\scriptscriptstyle {k-1}}_{j,t} \theta^{\scriptscriptstyle {h}}_{j,i} ) & \!\!\mbox{if } k \geq 1\\
0 & \!\!\mbox{if } k = \text{out}.
\end{cases}
\label{eq:rnn-pcn-dotx}
\end{equation}%

\smallskip 
\noindent\textbf{Prediction: }%
Given a sequential value
$S^{\text{in}}$ as input, every $\mu^{\scriptscriptstyle {k}}_{i}$ in the RNN is computed as the prediction via Eq.~\eqref{eq:rnn-forward-il}. Again, all error nodes converge to zero when $t\rightarrow \infty$, thus, $x^{\scriptscriptstyle {k}}_{i} = \mu^{\scriptscriptstyle {k}}_{i}$. 

\smallskip 
\noindent\textbf{Learning: }%
Given a sequential value
$S^{\text{in}}$ as input, the error in the output layer is set to  $\varepsilon^{\text{out}}_{i,0} = s^{\text{out}}_{i} - \mu^{\text{out}}_{i,0}$. From here, the inference phase spreads the error among all the neurons of the network. Once this process has converged to an equilibrium, the parameters of the network get updated in order to minimize the same objective function expressed in Eq.~\eqref{eq:pcn-f}. This causes the following weight updates:
\begin{equation}
    \begin{split}
        & \Delta \theta^x_{i,j} = \alpha \cdot {\textstyle\sum}_{k=1}^N \varepsilon^{k}_{i,t} s^{\text{in}}_{k,j} \\
        & \Delta \theta^h_{i,j} = \alpha \cdot {\textstyle\sum}_{k=1}^N \varepsilon^{k}_{i,t} f(x^{k-1}_j) \\
        & \Delta \theta^y_{i,j} = \alpha \cdot \varepsilon^{\text{out}}_{i,t} f(x^N_j).
    \end{split}
    \label{eq:weight_rnn_il}
\end{equation}
The derivations of Eqs.~\eqref{eq:rnn-pcn-dotx} and \eqref{eq:weight_rnn_il} are discussed in the supplementary material.

\subsection{Predictive Coding RNNs Trained with Z-IL}

We now show that Z-IL can also be carried over and scaled to RNNs, and that the equivalence of Theorem~\ref{thm:1} also holds for the considered RNNs. This equivalence can be extended to deeper networks, as it suffices to stack multiple layers (fully connected or convolutional) on top of the RNN's output layer.

\begin{theorem}
    \label{ther:final-equal-rnn} 
    Let $M$ be a recurrent PCN trained with Z-IL with $\gamma \,{=}\, 1$ and $\varepsilon^k_{i,0} \,{=}\, 0$ for $k \,{>}\,0$, and let $M'$ be its corresponding RNN, initialized as $M$ and trained with BP.
   Then, given the same sequential input $S \,{=}\, \{\bar s_1, \dots, \bar s_N\}$ to both, 
\begin{equation}
    \begin{split}
        & \Delta \theta^x_{i,j} = \Delta w^x_{i,j} \\
        & \Delta \theta^h_{i,j} = \Delta w^h_{i,j} \\
        & \Delta \theta^y_{i,j} = \Delta w^y_{i,j},
    \end{split}
    \label{eq:weight_rnn_theo}
\end{equation}
    for every $i,j>0$.
\end{theorem}

This concludes the main results of this work: our newly proposed algorithm Z-IL allows PCNs to do exact BP in complex models, such as RNNs and CNNs. We now move to a more robust analysis of Z-IL, showing practical advantages relative to IL that also trains PCNs.

\section{Computational Analysis of Z-IL}\label{sec:comp}

We now show that Z-IL is computationally significantly less costly than IL and only slightly more costly than BP.

By the previous sections, IL requires inference to be conducted for a number of steps until convergence, which is infinite theoretically and a fixed value empirically. In the only previous work that trains PCN-version of CNNs and RNNs, $T$ is between $100$ and $200$ \cite{millidge2020predictive}. In contrast, \mbox{Z-IL} requires inference to be conducted for $l_{\text{max}}$ steps, which is the smallest number of steps required to transfer information from the first to the last layer. Thus, IL requires significantly more steps of inference than Z-IL.


To have an empirical comparison on the computational complexity of IL and Z-IL, as well as to show that Z-IL creates only a little overhead compared to BP, Table \ref{tb:time} shows the average running time of each weights update of BP, IL, and Z-IL for different architectures. 
As can be seen, IL introduces large overheads, due to the fact that IL needs at least $T$ inference steps before conducting a weights update. 
In contrast, Z-IL runs with a minor overhead compared to BP, as it requires at most $l_{\text{max}}$ inference steps to complete one update of weights in all layers. To generate these numbers, we have performed several experiments on multiple datasets (FashionMNIST, ImageNet, and 8 different Atari games) and reported the averages. A comprehensive report of all the details of the experiments is in the supplementary material.

\begin{table}[t]
\caption{Average running time of each weights update (in ms) of BP, IL, and Z-IL for different architectures.}
  \label{tb:time}
 \vspace*{-1ex} \centering
\resizebox{0.9\columnwidth}{!}{
    \begin{tabular}{@{}cccc@{}}
    \toprule
    \cmidrule(r){1-4}
    Architecture & Backpropagation & Inference Learning \cite{millidge2020predictive} & Zero-Divergence Inference Learning                 \\
    \midrule
    ANN     & $3.72$  & $594.25$ & $3.81$ \\
    RNN     &  $5.64$ & $420.34$ & $5.67$ \\
    CNN     & $8.72$  & $661.53$ & $8.86$
    \\
    \bottomrule
    \end{tabular}
  }\vspace*{-2ex}
\end{table}

\section{Related Work}

PCNs are an influential theory of cortical function in theoretical and computational neuroscience with appealing interpretations, such as free-energy minimization 
\cite{bogacz2017tutorial}
and variational inference of  probabilistic models \cite{whittington2017approximation}.
They offer a single mechanism that accounts for diverse perceptual phenomena observed in the brain, such as end-stopping \cite{rao1999predictive}, repeti\-tion-suppression \cite{auksztulewicz2016repetition}, illusory motions 
\cite{watanabe2018illusory},
bistable perception,  
and even attentional modulation of neural activity 
\cite{kanai2015cerebral}.
Due to this solid biological grounding, PCNs are also attracting interest in machine learning recently, especially focusing on finding the links between PCNs and BP \cite{whittington2017approximation,millidge2020predictive}.

{Biologically plausible approximations to BP} have been intensively studied since the flourishing of BP, because on the one hand, the underlying principles of BP are unrealistic for an implementation in the brain
\cite{whittington2019theories}, 
but on the other hand, BP outperforms all alternative discovered frameworks~\cite{baldi2016theory} and closely reproduces activity patterns observed in the cortex
\cite{khaligh2014deep,yamins2016using,kell2018task}.
However, earlier biologically plausible approximations to BP were not scaling to larger and more complicated problems 
\cite{scellier2017equilibrium,illing2019biologically}.
More recent works show the capacity of scaling up biologically plausible approximations to the level of BP 
\cite{aljadeff2019cortical,wang2020supervised}.
But to date, none of the earlier or recent models 
has bridged the gaps at a degree of showing an equivalence to BP, though some of them \cite{lee2015difference,nokland2019training,ororbia2017learning,ororbia2019biologically} show that they  approximate (or are equivalent to) BP, but under unrealistic restrictions, e.g., the feedback is sufficiently weak \cite{xie2003equivalence,whittington2017approximation,millidge2020predictive}.

\section{Conclusion}

In this paper, we have extended existing results on Z-IL, an algorithm that trains PCNs, showing that it (i)  removes the unrealistic and non-trivial requirements of IL, while maintaining its biological plausibility, (ii) is strictly equivalent to BP on complex models such as CNNs and RNNs, and (iii) reduces the computational costs by a large margin compared  to IL, to a similar level of BP on all of the aforementioned models. These findings allow PCNs to reach the level of BP in terms of practical performance, as they are now provably able to train a network with the same accuracy of BP with only a small computational overhead.
Thus, the recently proposed Z-IL significantly strengthens the link between PCNs and BP, which is an important finding for both the machine learning and the neuroscience community. 
On the other hand, our work may also indicate that BP happens in the brain as just one part of the learning process, since Z-IL can be seen as a specific moment during the training of PCN with IL; thus, there may be something missing apart from BP.
So, BP may be more important in neuroscience than commonly thought.

 
\section*{Acknowledgments}
This work was supported by the China Scholarship Council under the State Scholarship Fund, by the
National Natural Science Foundation of China under the grant 61906063, by the Natural Science
Foundation of Tianjin City, China, under the grant 19JCQNJC00400, by the “100 Talents Plan” of
Hebei Province, China, under the grant E2019050017, and by the Medical Research Council UK
grant MC\_UU\_00003/1. This work was also supported by the Alan Turing Institute under the EPSRC
grant EP/N510129/1 and by the AXA Research Fund.

\bibliographystyle{abbrv}
\bibliography{references}

\appendix

\section{Additional Details about IL and Z-IL}

\subsection{\label{sec:bias-as-param}Bias Values as Parameters}

We now add bias values as parameters, denoted by $v^{\scriptscriptstyle {l+1}}_{i}$ and $\beta^{\scriptscriptstyle {l+1}}_{i}$ for ANNs and PCNs, respectively.

Formally, in ANNs trained with BP,  Eq.~\eqref{eq:bp-forward} becomes:
\begin{equation}
y^{\scriptscriptstyle {l}}_{i} = {\textstyle\sum}_{j=1}^{n^{\scriptscriptstyle {l+1}}} w^{\scriptscriptstyle {l+1}}_{i,j} f ( y^{\scriptscriptstyle {l+1}}_{j} )+v^{\scriptscriptstyle {l+1}}_{i}. 
\label{eq:bp-forward-bias}
\end{equation}
Accordingly, we have an update rule for bias parameters (bias update) in addition to the one of the weight parameters (weight update):
\begin{equation}
\Delta v^{\scriptscriptstyle {l+1}}_{i,j} 
= -\alpha\cdot {\partial E}/{\partial v^{\scriptscriptstyle {l+1}}_{i,j}}
=\alpha\cdot \delta^{\scriptscriptstyle {l}}_{i}.
\label{eq:bp-update-param-bias}
\end{equation}

Similarly, in PCNs trained with IL, originally recalled in the Preliminary Section becomes:
\begin{equation}
\mu^{\scriptscriptstyle {l}}_{i,t} = {\textstyle\sum}_{j=1}^{n^{\scriptscriptstyle {l+1}}} \theta^{\scriptscriptstyle {l+1}}_{i,j} f ( x^{\scriptscriptstyle {l+1}}_{j,t} ) + \beta^{\scriptscriptstyle {l+1}}_{i}
\mbox{ \ \ and \ \ }
\varepsilon^{\scriptscriptstyle {l}}_{i,t} = x^{\scriptscriptstyle {l}}_{i,t} - \mu^{\scriptscriptstyle {l}}_{i,t}.
\label{eq:pcn-forward-varepsilon-bias}
\end{equation}
Accordingly, we have an update rule for bias parameters in addition to the one of the weight parameters:
\begin{equation}
\Delta \beta^{\scriptscriptstyle {l+1}}_{i,j} = -\alpha\cdot {\partial F_{t}}/{\partial \beta^{\scriptscriptstyle {l+1}}_{i,j}} 
=  \alpha\cdot \varepsilon^{\scriptscriptstyle {l}}_{i,t}.
\label{eq:pcn-update-param-bias}
\end{equation}
Otherwise, all equations and conclusions of the main body still hold.

We only need to add the assumption that the bias parameter for both ANNs and PCNs are also identical initially. 
To prove the conclusion of zero divergence of the weights update, the procedure remains unchanged.
To prove the conclusion of zero divergence of the bias update, we only need Theorem~\ref{ther:final-equal-supp}, as it directly leads to the equivalence of bias update by  Eqs.~\eqref{eq:bp-update-param-bias} and~\eqref{eq:pcn-update-param-bias}.

\subsection{\label{sec:il-inference}Derivation of Eq.~\eqref{eq:pcn-dotx-il}}

We first expand Eq.~\eqref{eq:pcn-f} with the definition of $\varepsilon^{\scriptscriptstyle {l}}_{i,t} = x^{\scriptscriptstyle {l}}_{i,t} - \mu^{\scriptscriptstyle {l}}_{i,t}$:
\begin{align}
F_{t} = 
{\textstyle\sum}_{l=0}^{l_{\text{max}}-1} {\textstyle\sum}_{i=1}^{n^{\scriptscriptstyle {l}}}
{ \mbox{$\frac{1}{2}$} ( \varepsilon^{\scriptscriptstyle {l}}_{i,t} ) ^2}
= {\textstyle\sum}_{l=0}^{l_{\text{max}}-1} {\textstyle\sum}_{i=1}^{n^{\scriptscriptstyle {l}}}
{ \mbox{$\frac{1}{2}$} ( x^{\scriptscriptstyle {l}}_{i,t} - \mu^{\scriptscriptstyle {l}}_{i,t} ) ^2}\,.
\label{eq:pcn-f-expand-supp}
\end{align}
Inference minimizes $F_{t}$ by modifying $x^{\scriptscriptstyle {l}}_{i,t}$ proportionally to
the gradient of the objective function~$F_{t}$.
To calculate the derivative of $F_{t}$ over $x^{\scriptscriptstyle {l}}_{i,t}$, we note that each $x^{\scriptscriptstyle {l}}_{i,t}$ influences $F_{t}$ in two ways: (1)~it occurs in Eq.~\eqref{eq:pcn-f-expand-supp} explicitly, but (2) it also determines the values of $\mu^{\scriptscriptstyle {l-1}}_{k,t}$ via Eq.~\eqref{eq:pcn-forward-varepsilon}. 
Thus, the derivative contains two terms:
\begin{align*}
\Delta{x}^{\scriptscriptstyle {l}}_{i,t} 
&= - \gamma\cdot\frac{\partial F_{t}}{\partial x^{\scriptscriptstyle {l}}_{i,t}} \\
&= - \gamma\cdot(\frac{\partial \mbox{$\frac{1}{2}$} ( x^{\scriptscriptstyle {l}}_{i,t} - \mu^{\scriptscriptstyle {l}}_{i,t} ) ^2}{\partial x^{\scriptscriptstyle {l}}_{i,t}} + \frac{\partial {\textstyle\sum}_{k=1}^{n^{\scriptscriptstyle {l-1}}} \mbox{$\frac{1}{2}$} ( x^{\scriptscriptstyle {l-1}}_{k,t} - \mu^{\scriptscriptstyle {l-1}}_{k,t} ) ^2}{\partial x^{\scriptscriptstyle {l}}_{i,t}}) \\
&= \gamma\cdot ( -(x^{\scriptscriptstyle {l}}_{i,t} - \mu^{\scriptscriptstyle {l}}_{i,t}) + f' ( x^{\scriptscriptstyle {l}}_{i,t} )  {\textstyle\sum}_{k=1}^{n^{\scriptscriptstyle {l-1}}} (x^{\scriptscriptstyle {l-1}}_{k,t} - \mu^{\scriptscriptstyle {l-1}}_{k,t}) \theta^{\scriptscriptstyle {l}}_{k,i} ) \\
&= \gamma\cdot ( -\varepsilon^{\scriptscriptstyle {l}}_{i,t} + f' ( x^{\scriptscriptstyle {l}}_{i,t} ) {\textstyle\sum}_{k=1}^{n^{\scriptscriptstyle {l-1}}} \varepsilon^{\scriptscriptstyle {l-1}}_{k,t} \theta^{\scriptscriptstyle {l}}_{k,i} )\,.
\end{align*}
Considering also the special cases at $l=l_{\text{max}}$ and $l=0$, we obtain Eq.~\eqref{eq:pcn-dotx}.


\section{\label{sec:il-update}Derivation of Eq.~\eqref{eq:pcn-update-param}}

The weights update minimizes $F_{t}$ by modifying $\theta^{\scriptscriptstyle {l+1}}_{i,j}$ proportionally to the gradient of the objective function $F_{t}$.
To compute the derivative of the objective function $F_{t}$ over $\theta^{\scriptscriptstyle {l+1}}_{i,j}$, we note that $\theta^{\scriptscriptstyle {l+1}}_{i,j}$ affects the value of the function $F_{t}$ of Eq.~\eqref{eq:pcn-f-expand-supp} by influencing $\mu^{\scriptscriptstyle {l}}_{i,t}$ via Eq.~\eqref{eq:pcn-forward-varepsilon}, hence, 
\begin{align*}
\Delta \theta^{\scriptscriptstyle {l+1}}_{i,j} &= -\alpha\cdot {\partial F_{t}}/{\partial \theta^{\scriptscriptstyle {l+1}}_{i,j}} \\
&= -\alpha\cdot \frac{\partial \mbox{$\frac{1}{2}$} ( x^{\scriptscriptstyle {l}}_{i,t} - \mu^{\scriptscriptstyle {l}}_{i,t} ) ^2}{\partial \theta^{\scriptscriptstyle {l+1}}_{i,j}}  \\
&=  \alpha\cdot \varepsilon^{\scriptscriptstyle {l}}_{i,t} f ( x^{\scriptscriptstyle {l+1}}_{j,t} )\,.
\end{align*}

\subsection{Derivation of Kernel Updates of Convolutional PCNs}

The kernel update minimizes $F_{t}$ by modifying every entry of the kernel proportionally to the gradient of the objective function $F_{t}$. Recall that the weight matrix of a convolutional layer is a doubly-block circulant matrix whose non-zero entries are variables of $\lambda$. This is the final equation:

\begin{align}
    \Delta {\lambda}_a^{l+1} & = -\alpha \cdot \frac{\partial F_t}{\partial {\lambda}_a^{l+1}}  = {\textstyle\sum}_{(i,j) \in \mathcal C^{\scriptscriptstyle {l+1}}_a} \Delta \theta_{i,j}^{\scriptscriptstyle {l+1}}.
\end{align}
This equation is derived as follows:

\begin{align*}
\Delta \lambda^{\scriptscriptstyle {l+1}}_a &= -\alpha\cdot {\partial F_{t}}/{\partial \lambda^{l+1}_a} \\
&= -\alpha\cdot {\textstyle\sum}_l {\textstyle\sum}_a^k  \frac{\partial \mbox{$\frac{1}{2}$} ( x^{\scriptscriptstyle {l}}_{i,t} - \mu^{\scriptscriptstyle {l}}_{i,t} ) ^2}{\partial \lambda^{\scriptscriptstyle {l+1}}_a} \\
&=  -\alpha\cdot {\textstyle\sum}_l {\textstyle\sum}_a^k {\textstyle\sum}_{C^{l+1}_a} \frac{\partial \mbox{$\frac{1}{2}$} ( x^{\scriptscriptstyle {l}}_{i,t} - \mu^{\scriptscriptstyle {l}}_{i,t} ) ^2}{\partial \lambda^{\scriptscriptstyle {l+1}}_{a}} \\
&=  \alpha\cdot {\textstyle\sum}_{(i,j) \in C^{l+1}_a} \varepsilon^{\scriptscriptstyle {l}}_{i,t} f ( x^{\scriptscriptstyle {l+1}}_{j,t} ).
\end{align*}
The last equality  follows from the fact that a variation of $\lambda^l_a$ only affects the entries of the summation with indices in $C^{l+1}_a$. We now bring $\alpha$ inside the summation, and note that every term is equivalent to Eq.~\eqref{eq:pcn-update-param}. This gives
\begin{equation}
\Delta {\lambda}_a^{l+1} = {\textstyle\sum}_{(i,j) \in \mathcal C^{\scriptscriptstyle {l+1}}_a} \Delta \theta_{i,j}^{\scriptscriptstyle {l+1}}.
\end{equation}
\section{Derivation of Weight Updates of Recurrent PCNs}

The weight update minimizes $F_{t}$ by modifying every weight of the three matrices proportionally to the gradient of the objective function $F_{t}$. Particularly, for recurrent PCN, we have the following: 
\begin{equation}
    \begin{split}
        & \Delta \theta^x_{i,j} = \alpha \cdot {\textstyle\sum}_{k=1}^N \varepsilon^{k}_i s^{\text{in}}_j \\
        & \Delta \theta^h_{i,j} = \alpha \cdot {\textstyle\sum}_{k=1}^N \varepsilon^{k}_i f(x^{k-1}_j) \\
        & \Delta \theta^y_{i,j} = \alpha \cdot \varepsilon^{\text{out}}_i f(x^N_j).
        \label{eq:rnn-supp}
    \end{split}
\end{equation}

The derivation of $\Delta \theta^y_{i,j}$ follows directly from Eq.~\eqref{eq:pcn-update-param}, as it is a classical fully connected layer. On the other hand, $\Delta \theta^h_{i,j}$ and $\Delta \theta^x_{i,j}$ depend on different factors: every $\theta^h_{i,j}$ and $\theta^x_{i,j}$ influence the error $\varepsilon^k_j$ for every $k$ for every $k \leq N$. Hence, we have the following:

\begin{align*}
\Delta \theta^{h}_{i,j} &= -\alpha\cdot {\partial F_{t}}/{\partial \theta^{h}_{i,j}} \\
&= -\alpha\cdot {\textstyle\sum}_{k=1}^N \frac{\partial \mbox{$\frac{1}{2}$} ( x^{\scriptscriptstyle {k}}_{i,t} - \mu^{\scriptscriptstyle {k}}_{i,t} ) ^2}{\partial \theta^{h}_{i,j}}  \\
&=  \alpha\cdot {\textstyle\sum}_{k=1}^N \varepsilon^{k}_{i,t} f ( x^{\scriptscriptstyle {k+1}}_{j,t} )\,.
\end{align*}

\begin{align*}
\Delta \theta^{x}_{i,j} &= -\alpha\cdot {\partial F_{t}}/{\partial \theta^{x}_{i,j}} \\
&= -\alpha\cdot {\textstyle\sum}_{k=1}^N \frac{\partial \mbox{$\frac{1}{2}$} ( x^{\scriptscriptstyle {k}}_{i,t} - \mu^{\scriptscriptstyle {k}}_{i,t} ) ^2}{\partial \theta^{x}_{i,j}}  \\
&=  \alpha\cdot {\textstyle\sum}_{k=1}^N \varepsilon^{k}_{i,t} s^{\text{in}}_j\,.
\end{align*}

\section{Proof of Equivalence in CNNs}\label{sec:proof1}

In this section, we prove the following theorem, already stated in the main body of this work.

\begin{theorem}
    \label{ther:final-equal-supp} 
    Let $M$ be a convolutional PCN trained with Z-IL with $\gamma \,{=}\, 1$ and $\varepsilon^l_{i,0} \,{=}\, 0$ for $l\,{>}\,0$, and let $M'$ be its corresponding CNN, initialized as $M$ and trained with BP.
   Then, given the same datapoint $s$ to both networks, we have 
   \begin{equation}
    \Delta \theta^{\scriptscriptstyle {l}}_{i,j} = \Delta w^{\scriptscriptstyle {l}}_{i,j}  \ \ \ \mbox{and} \ \ \ \Delta {\lambda}^{l}_i = \ \Delta \rho^{l}_i,
    \end{equation}
    for every $i,j,l > 0$.
\end{theorem}

\begin{proof}[Proof]
A convolutional network is formed by a sequence of convolutional layers followed by a sequence of fully connected ones. First, we prove the following:

\medskip 
\noindent\textbf{Claim 1:} At $t\,{=}\,l$, we have $\varepsilon^{\scriptscriptstyle {l}}_{i,t} = \delta^{\scriptscriptstyle {l}}_{i}$. 

\medskip 
\noindent This first partial result is proven by induction on the depth $l$ of the two networks, and does not change whether the layer considered is convolutional or fully connected. For PCNs, as $t=l$, it is also inducing on the inference moments. We begin by noting that, in Z-IL, $\varepsilon^{\scriptscriptstyle {l}}_{i,t} =  \varepsilon^{\scriptscriptstyle {l}}_{i,l}$.
\begin{itemize}
    \item \emph{Base Case, $l=0$}:

        The condition $\varepsilon^l_{i,0} \,{=}\, 0$ gives us $\mu^{\scriptscriptstyle {l}}_{i,0}=y^{\scriptscriptstyle {l}}_{i}$. Placing this result into Eq.~\eqref{eq:pcn-forward-varepsilon-l0}
        and Eq.~\eqref{eq:delta-recursive}, we get
        $\varepsilon^{\scriptscriptstyle {l}}_{i,l}=\delta^{\scriptscriptstyle {l}}_{i}$.
    \item \emph{Induction Step:}. For $l \in \lbrace 1, \ldots , l_{\text{max}}-1 \rbrace$, we have: 
        \begin{align*}
            & \varepsilon^{\scriptscriptstyle {l}}_{i,l}= f' ( \mu^{\scriptscriptstyle {l}}_{i,0} ) {\textstyle\sum}_{k=1}^{n^{\scriptscriptstyle {l-1}}} \varepsilon^{\scriptscriptstyle {l-1}}_{k,l-1} \theta^{\scriptscriptstyle {l}}_{k,i}   \text{ \ \ \ by Lemma~\ref{lem:pcn-varepsilon-iterative-app}} \\
            & \delta^{\scriptscriptstyle {l}}_{i} =f' ( y^{\scriptscriptstyle {l}}_{i} ) {\textstyle\sum}_{k=1}^{n^{\scriptscriptstyle {l-1}}} \delta^{\scriptscriptstyle {l-1}}_{k} w^{\scriptscriptstyle {l}}_{k,i}  \text{\ \ \ \ \ \ \ \ \ \ by Eq.~\eqref{eq:delta-recursive}.}
            \label{eq:error}
        \end{align*}
    Furthermore, note that $w_{i,j}^{\scriptscriptstyle {l}}=\theta_{i,j}^{\scriptscriptstyle {l}}$, because of the same initialization of the network, and $\mu^{\scriptscriptstyle {l}}_{i,0} = y^{\scriptscriptstyle {l}}_{i}$, because of $\varepsilon^l_{i,0} \,{=}\, 0$ for $l\,{>}\,0$. Plugging these two equalities into the error equations above gives \begin{equation}
        \varepsilon^{\scriptscriptstyle {l}}_{i,l}=\delta^{\scriptscriptstyle {l}}_{i}, \ \ \text{if} \ \ \varepsilon^{\scriptscriptstyle {l-1}}_{k,l-1}=\delta^{\scriptscriptstyle {l-1}}_{k}.
    \end{equation}
    This concludes the induction step and proves the claim.
\end{itemize}

We now have to show the equivalence of the weights updates. We  start our study from fully connected layers.

\medskip 
\noindent\textbf{Claim 2:} We have  $\Delta \theta^{\scriptscriptstyle {l+1}}_{i,j} = \Delta w^{\scriptscriptstyle {l+1}}_{i,j}$ for every $i,j,l \geq 0$. 

\medskip 
\noindent Eqs.~\eqref{eq:pcn-update-param} and \eqref{eq:bp-update-param} state the following:

\begin{align*}
    & \Delta \theta^{\scriptscriptstyle {l+1}}_{i,j}
    = \alpha\cdot \varepsilon^{\scriptscriptstyle {l}}_{i,l} f ( x^{\scriptscriptstyle {l+1}}_{j,l} ), \\
    & \Delta w^{\scriptscriptstyle {l+1}}_{i,j} 
    =\alpha\cdot \delta^{\scriptscriptstyle {l}}_{i}
    f ( y^{\scriptscriptstyle {l+1}}_{j} ).
\end{align*}

Claim 1 gives $\varepsilon^{\scriptscriptstyle {l}}_{i,l}=\delta^{\scriptscriptstyle {l}}_{i}$.
We now have to show that ${\textstyle {\textstyle f ( x^{\scriptscriptstyle {l+1}}_{j,l} )=f ( y^{\scriptscriptstyle {l+1}}_{j} )}}$. 
The equivalence of the initial state between IL and BP gives ${x^{\scriptscriptstyle {l+1}}_{j,0} = \mu^{\scriptscriptstyle {l+1}}_{j,0} = \textstyle y^{\scriptscriptstyle {l+1}}_{j}}$.
Then, Lemma~\ref{lem:pcn-propagate-zero-app} shows that ${\textstyle x^{\scriptscriptstyle {l+1}}_{j,l}} = x^{\scriptscriptstyle {l+1}}_{j,0}$.
So,  ${\textstyle f ( x^{\scriptscriptstyle {l+1}}_{j,l} )=f ( y^{\scriptscriptstyle {l+1}}_{j} )}$.

\medskip 
\noindent\textbf{Claim 3:} We have  $\Delta \lambda^{\scriptscriptstyle {l+1}}_{a} = \Delta \rho^{\scriptscriptstyle {l+1}}_{a}$ for every $a,l \geq 0$. 

\medskip 
\noindent The law that regulates the updates of the kernels is given by the following equations:

\begin{align}\label{eq:cnn-proof}
    &\Delta {\lambda}_a^{l+1} = -\alpha \cdot \frac{\partial F_t}{\partial {\lambda}_a^{l}} =  {\textstyle\sum}_{(i,j) \in \mathcal C^{l+1}_a} \Delta \theta_{i,j}^{l} \\
    & \Delta \rho^{\scriptscriptstyle {l+1}}_{a} 
    = -\alpha\cdot {\partial E}/{\partial \rho^{\scriptscriptstyle {l+1}}_{a}}
    = {\textstyle\sum}_{(i,j) \in \mathcal C^{l+1}_a} \Delta w^{l+1}_{i,j}.
    \label{eq:cnn-proof-il}
\end{align}
These equations are equal if $\Delta \theta^{\scriptscriptstyle {l}}_{i,j} = \Delta w^{\scriptscriptstyle {l}}_{i,j}$ for every $i,j,l >0$, which is the result shown in Claim 2. Thus, the weight update at every iteration of Z-IL is equivalent to the one of BP for both convolutional and fully connected layers. 
\end{proof}

\begin{manuallemma}{A.3}
    \label{lem:pcn-propagate-zero-app}
    Let $M$ be a convolutional or recurrent PCN trained with Z-IL with $\gamma \,{=}\, 1$ and $\varepsilon^l_{i,0} \,{=}\, 0$ for $l\,{>}\,0$. Then, a variable $\bar x^l_{t}$ can only diverge from its corresponding initial state at time $t=l$. Formally,
    \begin{align*}
        & \overline{x}^{\scriptscriptstyle {l}}_{t<l}=\overline{x}^{\scriptscriptstyle {l}}_{0}, \overline{\varepsilon}^{\scriptscriptstyle {l}}_{t<l}=\overline{\varepsilon}^{\scriptscriptstyle {l}}_{0}=0, \overline{\mu}^{\scriptscriptstyle {l-1}}_{t<l}=\overline{\mu}^{\scriptscriptstyle {l-1}}_{0},  \text{ i.e.,} \\
        & \Delta{\overline{x}}^{\scriptscriptstyle {l}}_{t<l-1}=\overline{0}, \Delta{\overline{\varepsilon}}^{\scriptscriptstyle {l}}_{t<l-1}=\overline{0}, \Delta{\overline{\mu}}^{\scriptscriptstyle {l-1}}_{t<l-1}=\overline{0}
  \end{align*}
for  $l \in \lbrace 1, \ldots , l_{\text{max}}-1\rbrace$.
\end{manuallemma}

\begin{proof}[Proof]
    Starting from the inference moment $t=0$, $\overline{x}^{\scriptscriptstyle {0}}_{0}$ is dragged away from $\overline{\mu}^{\scriptscriptstyle {0}}_{0}$ and fixed to $\overline{s}^{\text{out}}$, i.e., $\overline{\varepsilon}^{\scriptscriptstyle {0}}_{0}$ turns into nonzero from zero. 
    Since $\overline{x}$ in each layer is updated only on the basis of $\overline{\varepsilon}$ in the same and previous adjacent layer, as indicated by Eq. \eqref{eq:pcn-dotx}, also considering that $\varepsilon^l_{i,0} \,{=}\, 0$, for all layers but the output layer, it will take $l$ time steps to modify $\overline{x}^{\scriptscriptstyle {l}}_{t}$ at layer $l$ from the initial state. 
    Hence, $\overline{x}^{\scriptscriptstyle {l}}_{t}$ will remain in that initial state $\overline{x}^{\scriptscriptstyle {l}}_{0}$ for all $t<l$, i.e., $\overline{x}^{\scriptscriptstyle {l}}_{t<l}=\overline{x}^{\scriptscriptstyle {l}}_{0}$.
    Furthermore, any change in $\overline{x}^{\scriptscriptstyle {l}}_{t}$ causes a change in $\overline{\varepsilon}^{\scriptscriptstyle {l}}_{t}$ and $\overline{\mu}^{\scriptscriptstyle {l-1}}_{t}$ instantly via Eq. \eqref{eq:pcn-forward-varepsilon} (otherwise $\overline{\varepsilon}^{\scriptscriptstyle {l}}_{t}$ and $\overline{\mu}^{\scriptscriptstyle {l-1}}_{t}$ remain in their corresponding initial states). 
    Thus, we know $\overline{\varepsilon}^{\scriptscriptstyle {l}}_{t<l}=\overline{\varepsilon}^{\scriptscriptstyle {l}}_{0}$ and $\overline{\mu}^{\scriptscriptstyle {l-1}}_{t<l}=\overline{\mu}^{\scriptscriptstyle {l-1}}_{0}$.
    Also, according to Eq.~\eqref{eq:pcn-dotx}, $\overline{\varepsilon}^{\scriptscriptstyle {l}}_{t<l}=\overline{\varepsilon}^{\scriptscriptstyle {l}}_{0}=0$.
    Equivalently, we have $\Delta{\overline{x}}^{\scriptscriptstyle {l}}_{t<l-1}=\overline{0}$, $\Delta{\overline{\varepsilon}}^{\scriptscriptstyle {l}}_{t<l-1}=\overline{0}$, and $\Delta{\overline{\mu}}^{\scriptscriptstyle {l-1}}_{t<l-1}=\overline{0}$.
\end{proof}

\begin{manuallemma}{A.4}
    \label{lem:pcn-varepsilon-iterative-app}
    Let $M$ be a convolutional PCN trained with Z-IL with $\gamma \,{=}\, 1$ and $\varepsilon^l_{i,0} \,{=}\, 0$ for $l\,{>}\,0$. Then, the prediction error $\varepsilon^{\scriptscriptstyle {l}}_{i,t}$ at $t=l$ (i.e., $\varepsilon^{\scriptscriptstyle {l}}_{i,l}$)  can be derived from itself at previous inference moments in the previous layer.
    Formally:
     \begin{align}
        \varepsilon^{\scriptscriptstyle {l}}_{i,l}= f' ( \mu^{\scriptscriptstyle {l}}_{i,0} ) {\textstyle\sum}_{k=1}^{n^{\scriptscriptstyle {l-1}}} \varepsilon^{\scriptscriptstyle {l-1}}_{k,l-1} \theta^{\scriptscriptstyle {l}}_{k,i} ,  \label{eq:pcn-varepsilon-iterative-app}
    \end{align}
for $l \in \lbrace 1, \ldots , l_{\text{max}}-1 \rbrace\,.$
\end{manuallemma}

\begin{proof}[Proof]
    We first write a dynamic version of 
    $\varepsilon^{\scriptscriptstyle {l}}_{i,t} = x^{\scriptscriptstyle {l}}_{i,t} - \mu^{\scriptscriptstyle {l}}_{i,t}$:
    \begin{equation}
        \varepsilon^{\scriptscriptstyle {l}}_{i,t} = \varepsilon^{\scriptscriptstyle {l}}_{i,t-1} + {(\Delta{x}^{\scriptscriptstyle {l}}_{i,t-1} - \Delta{\mu}^{\scriptscriptstyle {l}}_{i,t-1})\,,}
    \end{equation}
    where $\Delta{\mu}^{\scriptscriptstyle {l}}_{i,t-1}=\mu^{\scriptscriptstyle {l}}_{i,t}-\mu^{\scriptscriptstyle {l}}_{i,t-1}$.
    Then, we expand $\varepsilon^{\scriptscriptstyle {l}}_{i,l}$ with the above equation and simplify it with Lemma~\ref{lem:pcn-propagate-zero-app}, i.e., $\varepsilon^{\scriptscriptstyle {l}}_{i,t<l}=0$ and $\Delta{\mu}^{\scriptscriptstyle {l-1}}_{i,t<l-1}=0$:
    \begin{align}
        \varepsilon^{\scriptscriptstyle {l}}_{i,l} 
        = \varepsilon^{\scriptscriptstyle {l}}_{i,l-1} + {(\Delta{x}^{\scriptscriptstyle {l}}_{i,l-1} - \Delta{\mu}^{\scriptscriptstyle {l}}_{i,l-1})}
        ={\Delta{x}^{\scriptscriptstyle {l}}_{i,l-1}},. \label{eq:varepsilon-dotx}
    \end{align}
    for $l \in \lbrace 1, \ldots , l_{\text{max}}-1 \rbrace$. We further investigate $\Delta{x}^{\scriptscriptstyle {l}}_{i,l-1}$ expanded with the inference dynamic Eq.~\eqref{eq:pcn-dotx} and simplify it with Lemma~\ref{lem:pcn-propagate-zero-app}, i.e., $\varepsilon^{\scriptscriptstyle {l}}_{i,t<l}=0$,
    \begin{align}
            \Delta{x}^{\scriptscriptstyle {l}}_{i,l-1}
            = & \gamma ( -\varepsilon^{\scriptscriptstyle {l}}_{i,l-1} + f' ( x^{\scriptscriptstyle {l}}_{i,l-1} ) ){\textstyle\sum}_{k=1}^{n^{\scriptscriptstyle {l-1}}} \varepsilon^{\scriptscriptstyle {l-1}}_{k,l-1} \theta^{\scriptscriptstyle {l}}_{k,i} \\
            = &\gamma f' ( x^{\scriptscriptstyle {l}}_{i,l-1} ){\textstyle\sum}_{k=1}^{n^{\scriptscriptstyle {l-1}}} \varepsilon^{\scriptscriptstyle {l-1}}_{k,l-1} \theta^{\scriptscriptstyle {l}}_{k,i}, 
        \label{eq:dotx-itertive}
    \end{align}
    for $l \in \lbrace 1, \ldots , l_{\text{max}}-1 \rbrace$. Putting Eq.~\eqref{eq:dotx-itertive} into Eq.~\eqref{eq:varepsilon-dotx}, we obtain:
    \begin{align}
        \varepsilon^{\scriptscriptstyle {l}}_{i,l} 
        =\gamma f' ( x^{\scriptscriptstyle {l}}_{i,l-1} ){\textstyle\sum}_{k=1}^{n^{\scriptscriptstyle {l-1}}} \varepsilon^{\scriptscriptstyle {l-1}}_{k,l-1} \theta^{\scriptscriptstyle {l}}_{k,i},
    \end{align}
    for $l \in \lbrace 1, \ldots , l_{\text{max}}-1 \rbrace$. With Lemma~\ref{lem:pcn-propagate-zero-app}, $x^{\scriptscriptstyle {l}}_{i,l-1}$ can be replaced with $x^{\scriptscriptstyle {l}}_{i,0}$.
    With $\varepsilon^l_{i,0} \,{=}\, 0$ for $l\,{>}\,0$, we can further replace $x^{\scriptscriptstyle {l}}_{i,0}$ with $\mu^{\scriptscriptstyle {l}}_{i,0}$.
    Thus, the above equation becomes:
    \begin{align}
        \varepsilon^{\scriptscriptstyle {l}}_{i,l} 
        ={\gamma} f' ( \mu^{\scriptscriptstyle {l}}_{i,0} ){\textstyle\sum}_{k=1}^{n^{\scriptscriptstyle {l-1}}} \varepsilon^{\scriptscriptstyle {l-1}}_{k,l-1} \theta^{\scriptscriptstyle {l}}_{k,i}, \label{eq:pcn-varepsilon-iterative-app-with-gamma}
    \end{align}
    for $l \in \lbrace 1, \ldots , l_{\text{max}}-1 \rbrace$. Then, put $\gamma=1$, into the above equation.
\end{proof}

\section{Extension to the Case of Multiple Kernels per Layer}

In the theorem proved in the previous section, 
we have only considered CNNs with one kernel per layer. While networks of this kind are theoretically interesting, in practice a convolutional layer is made of multiple kernels. We now show that the result of Theorem~\ref{ther:final-equal-supp} still holds if we consider networks of this kind. Let $M_l$ be the number of kernels present in layer $l$. In Theorem~\ref{ther:final-equal-supp}, we have considered the case $M_l = 1$ for every convolutional layer. Consider now the following three cases: 
\begin{itemize}
    \item \textbf{Case 1: $M_l > 1, M_{l-1} = 1$}. We have a network with a convolutional layer at position $l$ with $M_l$ different kernels $\{ \bar \rho^{l,1}, \dots, \bar \rho^{l,M_l} \}$ of the same size $k$. The result of the convolution between the input $f(\bar y^{l})$ and a single kernel $\bar\rho^{l,m}$ is called \emph{channel}. The final output $\bar y^{l-1}$ of a convolutional layer is obtained by concatenating all the channels into a single vector. 
    The operation generated by convolutions and concatenation just described, can be written as a linear map $w^l \cdot f(\bar y^{l})$, where the matrix $w^l$ is formed by $M_l$ doubly-block circulant matrices stocked vertically, each of which has entries equal to the ones of a kernel $\bar \rho^{l,m}$. For each entry $\rho^{\scriptscriptstyle {l,m}}_a$ of each kernel in layer $l$, we denote by $\mathcal C^{\scriptscriptstyle {l}}_{m,a}$ the set of indices $(i,j)$ such that $w^{\scriptscriptstyle {l}}_{i,j} = \rho^{\scriptscriptstyle {l,m}}_a$. The equation describing the changes of parameters in the kernels is then the following:
    \begin{equation}
        \Delta \rho^{\scriptscriptstyle {l,m}}_{a} 
        = -\alpha\cdot {\partial E}/{\partial \rho^{\scriptscriptstyle {l,m}}_{a}}
        = {\textstyle\sum}_{(i,j) \in \mathcal C^{\scriptscriptstyle {l}}_{m,a}} \Delta w^{\scriptscriptstyle {l}}_{i,j}.
    \end{equation}
    
    \item \textbf{Case 2: $M_l = 1,M_{l-1} > 1$}. We now analyze what happens in a layer with only one kernel, when the input $f(\bar y^{l-1})$ comes from a layer with multiple kernels. This case differs from Case 1, because the input represents a concatenation of $M_{l-1}$ different channels. In fact, the kernel $\bar \rho^{l}$ gets convoluted with every channel independently. The resulting vectors of these convolutions are then summed together, obtaining $\bar y^{l}$. The operation generated by convolutions and summations just described, can be written as a linear map $w^l \cdot f(\bar y^{l})$. In this case, the matrix $w^l$ is formed by $M_{l-1}$ doubly-block circulant matrices stocked horizontally, each of which has entries equal to the ones of the kernel $\bar \rho^{l}$. For every entry $\rho^{\scriptscriptstyle {l}}_a$, we denote by $\mathcal C^{\scriptscriptstyle {l}}_{a}$ the set of indices $(i,j)$ such that $w^{\scriptscriptstyle {l}}_{i,j} = \rho^{\scriptscriptstyle {l}}_a$. The equation that describes the changes of parameters in the kernels is then the following:
    \begin{equation}
        \Delta \rho^{\scriptscriptstyle {l}}_{a} 
        = -\alpha\cdot {\partial E}/{\partial \rho^{\scriptscriptstyle {l}}_{a}}
        = {\textstyle\sum}_{(i,j) \in \mathcal C^{\scriptscriptstyle {l}}_{a}} \Delta w^{\scriptscriptstyle {l}}_{i,j}.
    \end{equation}
        
    \item \textbf{Case 3 (General Case): $M_l,M_{l-1} > 1$}. We now move to the most general case: a convolutional layer at position $l$ with $M_l$ different kernels $\{ \bar \rho^{l,1}, \dots, \bar \rho^{l,M_l} \}$, whose input $f(\bar y^l)$ is a vector formed by $M_{l-1}$ channels. In this case, every kernel does a convolution with every channel. The output $\bar y^{l+1}$ is obtained as follows: the results obtained using the same kernel on different channels are summed together, and concatenated with the results obtained using the other kernels. Again, this operation can be written as a linear map $w^l \cdot f(\bar y^{l})$. By merging the results obtained from Case 1 and Case 2, we have that the matrix $w^l$ is a grid of $M_l \times M_{l+1}$ doubly-block circulant submatrices. For every entry $\rho^{\scriptscriptstyle {l,m}}_a$ of every kernel in layer $l$, we denote by $\mathcal C^{\scriptscriptstyle {l}}_{m,a}$ the set of indices $(i,j)$ such that $w^{\scriptscriptstyle {l}}_{i,j} = \rho^{\scriptscriptstyle {l,m}}_a$. The equation  describing the changes of parameters in the kernels is then the following:
    \begin{equation}\label{eq:gen-cnn}
        \Delta \rho^{\scriptscriptstyle {l,m}}_{a} 
        = -\alpha\cdot {\partial E}/{\partial \rho^{\scriptscriptstyle {l,m}}_{a}}
        = {\textstyle\sum}_{(i,j) \in \mathcal C^{\scriptscriptstyle {l}}_{m,a}} \Delta w^{\scriptscriptstyle {l}}_{i,j}.
    \end{equation}
    
\end{itemize}

To integrate this general case in the proof of Theorem~\ref{ther:final-equal-supp}, it suffices to consider Eq.~\eqref{eq:gen-cnn}, and its equivalent formulation in the language of a convolutional PCN,
\begin{equation}
\Delta \rho^{\scriptscriptstyle {l+1}}_{a} 
    = -\alpha\cdot {\partial E}/{\partial \rho^{\scriptscriptstyle {l+1}}_{a}}
    = {\textstyle\sum}_{(i,j) \in \mathcal C^{l+1}_a} \Delta w^{l+1}_{i,j}
\end{equation}
 instead of Eqs.~\eqref{eq:cnn-proof} and \eqref{eq:cnn-proof-il}. Note that both equations are fully determined once we have computed $\Delta w^{\scriptscriptstyle {l}}_{i,j}$ and $\Delta \theta^{\scriptscriptstyle {l}}_{i,j}$ for every $i,j>0$. Hence, the result follows directly by doing the same computations.

\section{\label{sec:relax-gamma}Solely Relaxing  $\gamma \,{=}\, 1$}

Solely relaxing $\gamma \,{=}\, 1$ will change the result of Lemma~\ref{lem:pcn-varepsilon-iterative-app}. Particularly, we would have Eq.~\eqref{eq:pcn-varepsilon-iterative-app} changing to:
\begin{align}
    \varepsilon^{\scriptscriptstyle {l}}_{i,l}={\gamma}f' ( \mu^{\scriptscriptstyle {l}}_{i,0} ) {\textstyle\sum}_{k=1}^{n^{\scriptscriptstyle {l-1}}} \varepsilon^{\scriptscriptstyle {l-1}}_{k,l-1} \theta^{\scriptscriptstyle {l}}_{k,i},
\end{align}
for $l \in \lbrace 1, \ldots , l_{\text{max}}-1 \rbrace$.
Since the derivation of Lemma~\ref{lem:pcn-varepsilon-iterative-app} terminates at Eq.~\eqref{eq:pcn-varepsilon-iterative-app-with-gamma}.
It further causes the conclusion of Theorem~\ref{ther:final-equal-supp}  changing from $\varepsilon^{\scriptscriptstyle {l}}_{i,t}=\delta^{\scriptscriptstyle {l}}_{i}$ to $\varepsilon^{\scriptscriptstyle {l}}_{i,t}=\gamma^{l}\delta^{\scriptscriptstyle {l}}_{i}$ at $t=l$, the proof of which is the same as that of the original Theorem~\ref{ther:final-equal-supp}  but using  Lemma~\ref{lem:pcn-varepsilon-iterative-app}.
This changes the conclusion of Theorem~\ref{ther:final-equal-supp}  from
\begin{equation}
    {\partial F_{t}}/{\partial \theta^{\scriptscriptstyle {l+1}}_{i,j}}={\partial E}/{\partial w^{\scriptscriptstyle {l+1}}_{i,j}}
\end{equation}
to
\begin{equation}
{\partial F_{t}}/{\partial \theta^{\scriptscriptstyle {l+1}}_{i,j}}=\gamma^{l}{\partial E}/{\partial w^{\scriptscriptstyle {l+1}}_{i,j}}, 
\end{equation}
where $t=l$.

Thus, solely relaxing the condition $\gamma=1$ results in BP with a different learning rate for different layers, where $\gamma$ is the decay factor of this learning rate along layers.

\section{Proof of Equivalence in RNNs}\label{sec:proof2}

In this section, we  prove the following theorem, already stated in the main body of this work.

\begin{theorem}
    \label{ther:final-equal-rnn-supp} 
    Let $M$ be a recurrent PCN trained with Z-IL with $\gamma \,{=}\, 1$ and $\varepsilon^k_{i,0} \,{=}\, 0$ for $k \,{>}\,0$, and let $M'$ be its corresponding RNN, initialized as $M$ and trained with BP.
   Then, given the same sequential input $S \,{=}\, \{\bar s_1, \dots, \bar s_N\}$ to both networks, we have 
\begin{equation}
    \begin{split}
        & \Delta \theta^x_{i,j} = \Delta w^x_{i,j} \\
        & \Delta \theta^h_{i,j} = \Delta w^h_{i,j} \\
        & \Delta \theta^y_{i,j} = \Delta w^y_{i,j},
    \end{split}
\end{equation}
    for every $i,j>0$.
\end{theorem}

\begin{proof}

The network $M$ has depth $2$; hence, we set $T=2$. We now prove the following three equivalences: $(1)$ $\Delta \theta^y = \Delta w^y$, $(2)$ $\Delta \theta^h = \Delta w^h$, and $(3)$ $\Delta \theta^x = \Delta w^x$.

The proof of $(1)$ is straightforward, since both the output layers $\theta^y$ and $w^y$ are fully connected. Particularly, we have already shown the equivalence for this kind of layers in Theorem~\ref{ther:final-equal-supp}. Before proving $(2)$ and $(3)$, we show an intermediate result needed in both cases.

\medskip 
\noindent \textbf{Claim:} Given a sequential input $S^{in}$ of length $N$, at $t=1$ we have $\varepsilon^k_{i,1} = \delta^k_i$ for every $k \leq N$. 

\medskip 
\noindent This part of the proof is done by induction on $N$. 
\begin{itemize}
    \item Base Case: $N = 1$. Given a sequential input of length $1$, we have a fully connected network of depth $2$ with $w^1 = w^y$ (resp. $\theta^1 = \theta^y$) and $w^2 = w^x$ (resp. $\theta^2 = \theta^x$). We have already proved this result in Theorem~\ref{ther:final-equal-supp}.
    
    \item Induction Step. Let us assume that, given a sequential input $S^{in}$ of length $N$, the claim $\varepsilon^k_{i,1} = \delta^k_i$ holds for every $k \in \{1, \dots, N\}$. Let us now assume we have a sequential input of length $N+1$. Note that the errors  $\varepsilon^k_{i,1}$ and  $\delta^k_i$ are computed backwards starting from $k=N+1$. Hence, the quantities $\varepsilon^k_{i,1}$ and  $\delta^k_i$ for $k \in \{2, \dots, N+1\}$ are computed as they were the errors of a sequential input of length $N$. It follows by the induction argument that $\varepsilon^k_{i,1} = \delta^k_i$ for every $k \in \{2, \dots, N+1\}$. To conclude the proof, we have to show that $\varepsilon^1_{i,1} = \delta^1_i$.  For $k=1$, we have:
            \begin{align*}
            & \varepsilon^{\scriptscriptstyle {1}}_{i,l}= f' ( \mu^{\scriptscriptstyle {1}}_{i,0} ) {\textstyle\sum}_{j=1}^{n} \varepsilon^{\scriptscriptstyle {2}}_{j,t} \theta^{\scriptscriptstyle {h}}_{j,i}   \text{ \ \ \ \ \ \ \ by Lemma~\ref{lem:pcn-varepsilon-iterative-app-rnn}} \\
            & \delta^{\scriptscriptstyle {1}}_{i} = f' ( y^{\scriptscriptstyle {1}}_{i} ) {\textstyle\sum}_{j=1}^{n} \delta^{\scriptscriptstyle {2}}_{j} w^{\scriptscriptstyle {h}}_{j,i}.  \text{\ \ \ \ \ \ \ \ \ \ \  by Eq.~\eqref{eq:error-rnn}.}
            \label{eq:error-rnn-proof}
        \end{align*}
    
    Note that $w^h_{i,j}=\theta^h_{i,j}$, because of the same initialization of the network. Furthermore, $\mu^{\scriptscriptstyle {k}}_{i,0} = y^{\scriptscriptstyle {k}}_{i}$ for every $k$ because of $\varepsilon^k_{i,0} \,{=}\, 0$. Plugging these two equalities into the error equations above gives $\varepsilon^1_{i,1} = \delta^1_i$.
    This concludes the induction step and proves the claim.
\end{itemize}

\smallskip\noindent\emph{(2)  $\Delta \theta^h = \Delta w^h$.} Recall that Eqs.~\eqref{eq:weight_rnn_il} and \eqref{eq:weight_rnn_bp} state that
\begin{align*}
    & \Delta \theta^h_{i,j} = \alpha \cdot {\textstyle\sum}_{k=1}^N \varepsilon^{k}_{i,t} f(x^{k-1}_{j,1}) \\
    & \Delta w^h_{i,j} = \alpha \cdot {\textstyle\sum}_{k=1}^N \delta^{k}_i f(y^{k-1}_j).
\end{align*}
The claim shown above gives $\varepsilon^{k}_{i,1} = \delta^{k}_i$.  We thus have to show that $x^{k}_{j,1} = y^{k}_j$. The condition $\varepsilon^k_{j,0} = 0$ gives $x^{k}_{j,0} = \mu^{k}_{j,0} = y^{k}_{j}$. Moreover, by Lemma~\ref{lem:pcn-varepsilon-iterative-app}, $x^{k}_{j,1} = x^{k}_{j,0}$. So, $x^{k}_{j,1} = y^{k}_j$. 

\smallskip\noindent\emph{(3)  $\Delta \theta^x = \Delta w^x$.} Recall that Eqs.~\eqref{eq:weight_rnn_il} and \eqref{eq:weight_rnn_bp} state that
\begin{align*}
    & \Delta \theta^x_{i,j} = \alpha \cdot {\textstyle\sum}_{k=1}^N \varepsilon^{k}_{i,t} s^{\text{in}}_{k,j} \\
    & \Delta w^x_{i,j} = \alpha \cdot {\textstyle\sum}_{k=1}^N \delta^{k}_i s^{\text{in}}_{k,j}.
\end{align*}
The equality $\Delta \theta^x = \Delta w^x$  directly follows from $\varepsilon^{k}_{i,1} = \delta^{k}_i$.
\end{proof}

\begin{manuallemma}{A.5}
    \label{lem:pcn-varepsilon-iterative-app-rnn}
    Let $M$ be a recurrent PCN trained with Z-IL on a sequential input $S^{in}$ of length $N$. Furthermore, let us assume that $\gamma \,{=}\, 1$ and $\varepsilon^k_{i,0} \,{=}\, 0$ for every $k \in \{1, \dots, N\}$. Then, the prediction error $\varepsilon^{\scriptscriptstyle {k}}_{i,t}$ at $t=1$ (i.e., $\varepsilon^{\scriptscriptstyle {k}}_{i,1}$)  can be derived from the previous recurrent layer.
    Formally:
     \begin{align}
        \varepsilon^{\scriptscriptstyle {k}}_{i,1}= f' ( \mu^{\scriptscriptstyle {k}}_{i,0} ) {\textstyle\sum}_{j=1}^{n^{\scriptscriptstyle {k+1}}} \varepsilon^{\scriptscriptstyle {k+1}}_{j,1} \theta^{\scriptscriptstyle {h}}_{j,i} ,  
    \end{align}
    for $k \in \lbrace 1, \ldots , N-1 \rbrace\,.$
\end{manuallemma}

\begin{proof}
    Equivalent to the one of Lemma~\ref{lem:pcn-varepsilon-iterative-app}. The only difference is that in Lemma~\ref{lem:pcn-varepsilon-iterative-app} we iterate over the previous layer $l$ at time $t=l$, while here the iterations happen over the previous recurrent layer $k$ at fixed time $t=1$.
\end{proof}

\section{Experiments for Theorem Discovery}\label{sec:val}

\begin{table}
  \caption{Euclidean distance of the weights after one training step of Z-IL (and variations), and BP.}
  
  \label{tb:time}\vspace*{-1ex}
  \centering
  \resizebox{0.9\columnwidth}{!}{
    \begin{tabular}{@{}ccccc@{}}
    \toprule
    \cmidrule(r){1-4}
    Model & Z-IL  & Z-IL without Layer-dependent Update & Z-IL with $\varepsilon^l_{i,0} \neq 0$ & Z-IL with $\gamma = 0.5$\\
    \midrule
    ANN     & $0$ & $1.42\times10^2 $   & $7.22$ & $8.67 \times 10^4$ \\
    RNN     & $0$ & $6.05\times10^3$   & $9.60$ & $6.91\times10^5$ \\
    CNN     & $0$ & $5.93\times10^5$   & $7.93\times10^2$ & $9.87\times10^8 $ \\
    \bottomrule
    \end{tabular}
  }\vspace*{1.5ex}
  \label{tb:abl}
\end{table}

To search for the conditions necessary to reach equivalence, we have performed a vast amount of experiments, which we report for expository purposes. On the same  MLPs, CNNs, and RNNs used for the other experiments on various datasets, we show that all the conditions of Z-IL are needed to obtain exact backpropagation. Particularly, by starting from the same weight initialization, we have conducted one training step of five different learning algorithms:
\begin{enumerate}
    \item BP,
    \item Z-IL,
    \item Z-IL without layer-dependent update,
    \item Z-IL with $\varepsilon^l_{i,0} \neq 0$,
    \item Z-IL with $\gamma = 0.5$.
\end{enumerate}
Note that the last three algorithms are the variations of Z-IL obtained ablating each one of the initial conditions. This was useful to check whether they were all needed and necessary to obtain our exactness result. 

After conducting one training step of each algorithm, we have computed the Euclidean distance between the weights obtained by one of the Algorithms $1-4$, and the ones obtained by BP. The results of these experiments, reported in Table~\ref{tb:abl}, show that the all the three conditions of Z-IL are necessary in order to achieve zero divergence with BP.
To provide full evidence of the validation of our theoretical results, we have conducted this experiment using ANNs, CNNs, and RNNs.

Further details about the experiments can be found in the section below.

\section{Reproducibility of the Experiments} 
In this section, we provide the details of all the experiments shown in Sections~\ref{sec:comp} and \ref{sec:val}.

\medskip 
\noindent\textbf{ANNs:} 
To perform our experiments with fully connected networks, we have trained three architectures with different depth on FashionMNIST. Particularly, these networks have an hidden dimension of $128$ neurons, and $2,3$ and $4$ layers, respectively. Furthermore, we have used a  batch of $20$ training points, and a learning rate of $0.01$. The numbers reported for  both the experiments are the averages runs over the three architectures.

\medskip 
\noindent\textbf{CNNs:} For our complexity experiments on convolutional networks, we have used Alexnet trained on both FashionMNIST and ImageNET. As above, we have used a  batch of $20$ training points, a learning rate of $0.01$ and reported the average of the experiments over the two datasets. 
For out full-training experiments, we have used Alexnet trained on CIFAR10. As hyperparameters, we have used a learning rate of $0.1$ and a batch size of $128$. Furthermore, we have trained the network for $120$ epochs and reported the best early stopping accuracy.

\medskip 
\noindent\textbf{RNNs:} To conclude, we have trained a reinforcement learning agent on a single-layer many to one RNN, with $n = n^{out} = 128$, on eight different Atari games. Batch size and learning rate are $32$ and $0.001$, respectively. Again, the reported results are the average of all the experiments performed on this architecture.

All experiments are averaged over 5 seeds and conducted on 2 Nvidia GeForce GTX 1080Ti  GPUs and 8 Intel Core i7 CPUs, with 32 GB RAM. Furthermore, to avoid rounding errors, we have initialized the weights in \emph{float32}, and then transformed them in \emph{float64}.

\end{document}